\documentclass{article}
\usepackage{graphicx} % more modern
\usepackage[tight]{subfigure}
\usepackage{times}
\usepackage{natbib}

\usepackage{algorithm}
\usepackage{algorithmic}
\usepackage{hyperref}

\usepackage[accepted]{icml2014}
\usepackage{savesym}
\usepackage{multirow}
\usepackage{url}
\usepackage[titletoc]{appendix}
\usepackage{xspace}
\usepackage{url}
\usepackage{hyperref}
\usepackage{enumitem}
\usepackage{color}
\usepackage[all=normal,floats,wordspacing,paragraphs]{savetrees}
\usepackage[tight]{subfigure}
\usepackage{graphicx}
\usepackage{amsmath,amssymb,amsthm,xfrac}
\setitemize{noitemsep,topsep=1pt,parsep=1pt,partopsep=1pt, leftmargin=12pt}

\usepackage{subfloat}
\newtheorem{lemma}{Lemma}

\newtheorem{theorem}{Theorem}
\newtheorem*{theorem*}{Theorem}
\newtheorem{proposition}{Proposition}

\floatname{algorithm}{Policy}

\makeatletter

\newcommand{\Rmnum}[1]{\expandafter\@slowromancap\romannumeral #1@}

\newcommand{\X}{\mathcal{X}}
\newcommand{\Hyp}{\mathcal{H}}
\newcommand{\SetHyp}{\{h_1, h_2, \dots, h_n\}}
\newcommand{\real}{\mathbb{R}}

\DeclareMathOperator{\sgn}{sgn}
\newcommand{\D}{\mathcal{D}}
\DeclareMathOperator*{\argmax}{\arg\max}
\DeclareMathOperator*{\argmin}{\arg\min}
\DeclareMathOperator{\err}{err}
\DeclareMathOperator{\OPT}{OPT}
\newcommand{\lerr}{\err_L}
\newcommand{\E}{\mathbb{E}}
\newcommand{\wl}{w_l}
\newcommand{\wt}{w_o}

\newcommand{\xii}{\xi_i}
\newcommand{\NF}[1]{Z_{#1}}
\newcommand{\pl}[2]{P_{#2}(#1)}
\newcommand{\pt}[2]{P^{(t)}_{#2}(#1)}
\newcommand{\polytop}{\Pi}
\newcommand{\poly}{\mathcal{P}}

\newcommand{\Alggbs}{{\textsc{RGTP}}\xspace}
\newcommand{\prior}{P_0}
\newcommand{\BZ}{\eta}
\newcommand{\BZt}{\eta^{(t)}}

\newcommand{\filter}{\mathcal{F}}
\newcommand{\filterg}{\mathcal{G}}
\newcommand{\wph}{\delta}
\newcommand{\ppp}[1]{\delta^+_{#1}}
\newcommand{\ie}{i.e.}
\newcommand{\eg}{e.g.}
\newcommand{\algo}{{\sc STRICT}\xspace}
\newcommand{\ugtp}{{\em Submodular Teaching for cRowdsourcIng ClassificaTion}\xspace}

\icmltitlerunning{Near-Optimally Teaching the Crowd to Classify}
\begin{document}

% Surrogate Teaching for Crowdsourcing Classification
% Possible acronyms:
%
%CRITIC = CRowdsourcIng TeachIng Classifications
%STANCE = Surrogate TeAchiNg ClassifiErs
%STACCATO = Submodular TeAching for Crowdsourcing ClassificATiOn
%TRACY = TRAining to ClassifY
%STRICT = Submodular TRaIning for ClassificaTion
%STRICT = Submodular Teaching for cRowdsourcIng ClassificaTion
%TACO = TeAching ClassificatiOn

\twocolumn[
\icmltitle{Near-Optimally Teaching the Crowd to Classify}
%\icmltitle{Near-Optimal Crowd Teaching}

% It is OKAY to include author information, even for blind
% submissions: the style file will automatically remove it for you
% unless you've provided the [accepted] option to the icml2013
% package.
\icmlauthor{Adish Singla}{adish.singla@inf.ethz.ch}
\icmlauthor{Ilija Bogunovic}{ilija.bogunovic@epfl.ch}
\icmlauthor{Gabor Bartok}{bartok@inf.ethz.ch}
\icmlauthor{Amin Karbasi}{amin.karbasi@inf.ethz.ch}
\icmlauthor{Andreas Krause}{krausea@ethz.ch}
\icmladdress{ETH Zurich, Universit\"atstrasse 6, 8092 Z\"{u}rich, Switzerland}

% You may provide any keywords that you
% find helpful for describing your paper; these are used to populate
% the "keywords" metadata in the PDF but will not be shown in the document
\icmlkeywords{Active teaching, human learning, crowdsourcing, machine learning}

\vskip 0.3in
]
%%%%%%%%%%%%%%%%%%%%%%%%%%%%%%%%%%%%%%%%%%%%%%%%%%%%%%%%%%%%%%%%%%%%%%%%%%
%%%%%%%%%%%%%%%%%%%%%%%%%%%%%%%%%%%%%%%%%%%%%%%%%%%%%%%%%%%%%%%%%%%%%%%%%%
\begin{abstract}
How should we present training examples to learners to teach them classification rules?
This is a natural problem when training workers for crowdsourcing labeling tasks, and is also motivated by challenges in data-driven online education.
%Natural settings where this problem occurs are crowdsourcing labeling tasks as well as teaching concepts in massive online education.
%
%Is it possible to teach workers while crowdsourcing classification tasks? Amongst the challenges: (a) workers have different (unknown) skills, competence, and learning rate to which the teaching  must be adapted, (b) feedback on the workers' progress is limited, (c) we may not have informative features for our data (otherwise crowdsourcing may be unnecessary).
%the underlying learning process used by the workers is unknown.
We propose a natural stochastic model of the learners, modeling them as randomly switching among hypotheses based on observed feedback. We then develop \algo, an efficient algorithm for selecting examples to teach to workers.
%
%show how a teaching system can exploit this model to interactively teach the workers. 
Our solution greedily maximizes a submodular surrogate objective function in order to select examples to show to the learners. We prove that our strategy is competitive with the optimal teaching policy. Moreover, for the special case of linear separators, we prove that an exponential reduction in error probability can be achieved.
Our experiments on simulated workers as well as three real image annotation tasks on Amazon Mechanical Turk show the effectiveness of our teaching algorithm.
\end{abstract}
%\andreas{?} air quality monitoring
%%%%%%%%%%%%%%%%%%%%%%%%%%%%%%%%%%%%%%%%%%%%%%%%%%%%%%%%%%%%%%%%%%%%%%%%%%
%%%%%%%%%%%%%%%%%%%%%%%%%%%%%%%%%%%%%%%%%%%%%%%%%%%%%%%%%%%%%%%%%%%%%%%%%%
% !TEX root =  teaching.tex
%%%%%%%%%%%%%%%%%%%%%%%%%%%%%%%%%%%%%%%%%%%%%%%%%%%%%%%%%%%%%%%%%%%%%%%%%%
%%%%%%%%%%%%%%%%%%%%%%%%%%%%%%%%%%%%%%%%%%%%%%%%%%%%%%%%%%%%%%%%%%%%%%%%%%
\vspace{-7mm}
\section{Introduction}\label{sec:introduction}
\vspace{-1mm}
Crowdsourcing services, such as Amazon's Mechanical Turk platform\footnote{MTurk: \url{https://www.mturk.com/mturk/welcome}} (henceforth MTurk), are becoming vital for outsourcing information processing to large groups of workers. %, include image annotation, transcription of audio and rating the relevance of web pages.
Machine learning, AI, and citizen science systems can hugely benefit from the use of these services as large-scale annotated data is often of crucial importance \cite{2008-emnlp_cheap-fast-good-annotations,2008-cvpr_utility-annotation,galaxyzoo}. Data collected from such services however is often noisy, e.g., due to spamming, inexpert or careless workers \cite{2008-cvpr_utility-annotation}. As the accuracy of the annotated data is often crucial, the problem of tackling noise from crowdsourcing services has received considerable attention. Most of the work so far has focused on methods for combining labels from many annotators \cite{2010-nips_crowd-wisdom,2011-nips_krause_crowdclustering,2013-www_aggregating-ratings}
%by often jointly estimating worker's realibility and data quality . Often, gold st
or in designing control measures by estimating the worker's reliabilities through ``gold standard'' questions \cite{2008-emnlp_cheap-fast-good-annotations}. %We take a fundamentally different approach in solving this problem by exploring the design of intelligent systems which can effectively teach the workers.

%%%%%%%%%%%%%%%%%%%%%%%%%%%%%%%%%%%%%%%%%%%%%%%%%%%%%%%%%%%%%%%%%%%%%%%%%%
%\subsection{Our approach} \label{sec:introduction:approach}
In this paper, we explore an orthogonal direction: {\em can we teach workers in crowdsourcing services in order to improve their accuracy?} That is, instead of designing models and methods for determining workers' reliability, can we develop intelligent systems that teach workers to be more effective?  While we focus on crowdsourcing in this paper, similar challenges arise in other areas of data-driven education. As running examples, in this paper we focus on crowdsourcing image labeling. In particular, we consider the task of classifying animal species, an important component in several citizen science projects such as the eBird project \cite{sullivan2009ebird}.

We start with a high-level overview of our approach.  Suppose we wish to teach the crowd to label a large set of images (e.g., distinguishing butterflies from moths). How can this be done without already having access to the labels, or a set of informative features, for all the images (in which case crowdsourcing would be useless)? We suppose we have ground truth labels only for a small ``teaching set'' of examples. Our premise is that if we can teach a worker to classify this teaching set well, she can generalize to new images. In our approach, we first elicit---on the teaching set---a set of candidate features as well as a collection of hypotheses (e.g., linear classifiers) that the crowd may be using. We will describe the concrete procedure used in our experimental setup in Section~\ref{sec:experiments_setup}.
%In our approach, we first ask a small set of workers to each label the entire teaching set. From those labels, we use an existing Bayesian model \cite{2010-nips_crowd-wisdom} to infer a set of candidate features as well as hypotheses (linear classifiers) that the crowd may be using.
Having access to this information %the true labels and features for the teaching examples, as well as the crowd's prior distribution over hypotheses,
we use a teaching algorithm to select training examples and steer the learner towards the target hypothesis.

%An overview of our approach is given in Figure~\ref{}.
%%%%%%%%%%%%%%%%%%%%%%%%%%%%%%%%%%%%%%%%%%%%%%%%%%%%%%%%%%%%%%%%%%%%%%%%%%
%%%%%%%%%%%%%%%%%%%%%%%%%%%%%%%%%%%%%%%%%%%%%%%%%%
\begin{figure*}[t]
\centering
\includegraphics[width=0.7\textwidth]{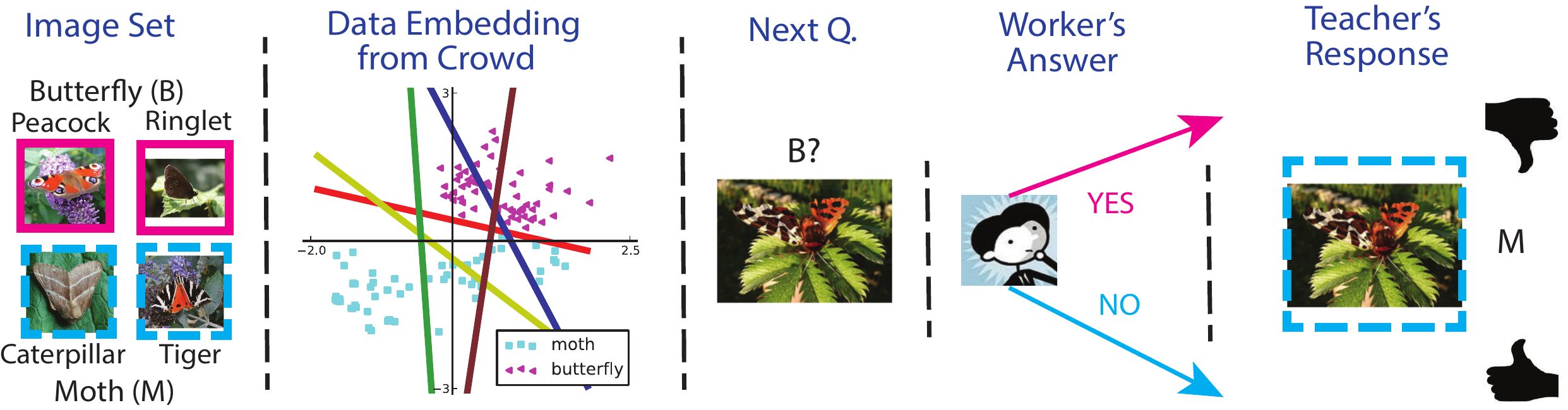}
\vspace{-4mm}
\caption{Illustration of crowd-teaching. Given a large set of images, % (e.g., of ``butterflies'' and ``moths''),
the teacher (randomly) picks a small ``teaching set''. For this set,
%first obtains
expert labels, as well as candidate features and hypotheses used by the crowd are elicited (see Section~\ref{sec:experiments_setup}).
%an embedding for small ``teaching set'' of images, as well as potential hypotheses used from a collection of workers.
The teacher then uses this information to teach the rest of the crowd to label the rest of the data, for which no features or labels are available. The teacher sequentially provides an unlabeled example from the teaching set to the worker, who attempts an answer. Upon receipt of the correct label, the learner may update her hypothesis before the next example is shown.}
\vspace{-4mm}
\label{fig:hit}
\end{figure*}
%%%%%%%%%%%%%%%%%%%%%%%%%%%%%%%%%%%%%%%%%%%%%%%%%%
%%%%%%%%%%%%%%%%%%%%%%%%%%%%%%%%%%%%%%%%%%%%%%%%%%%%%%%%%%%%%%%%%%%%%%%%%%
Classical work on teaching classifiers (reviewed in Section~\ref{sec:related}), assumes that learners are {\em noise-free}: Hypotheses are immediately eliminated from consideration upon observation of an inconsistent training example. As we see in our experiments (Section~\ref{sec:experiments_results}), such approaches can be brittle. In contrast, we propose a {\em noise-tolerant} stochastic model of the learners, capturing our assumptions on how they incorporate training examples. We then (Section~\ref{sec:teacher}) propose \algo (\ugtp), a novel teaching algorithm that selects a sequence of training examples to the workers in order to steer them towards the true hypothesis. We theoretically analyze our approach, proving strong approximation guarantees and teaching complexity results. Lastly, we demonstrate the effectiveness of our model and \algo policy on three real image annotation tasks, carried out on the MTurk platform.

\vspace{-2mm}
\section{Background and Teaching Process} \label{sec:protocol}
%%%%%%%%%%%%%%%%%%%%%%%%%%%%%%%%%%%%%%%%%%%%%%%%%%%%%%%%%%%%%%%%%%%%%%%%%%
%In the following, we first describe the learning domain and the teaching protocol. Introduce the generic and abstract model of interaction during the teaching process.
We %start with an overview of our approach, and then
now describe our learning domain and teaching protocol. As a running example, we consider the task of teaching to classify images, e.g., to distinguish butterflies from moths (see Figure~\ref{fig:hit}).
\vspace{-3mm}
\subsection{The domain and the teaching protocol}
\vspace{-1mm}
%As mentioned in the introduction \gabor{(?)}, the main focus of this paper is teaching a binary classification hypothesis to a learner.
Let $\X$ denote a set of examples (\eg, images), called the {\em teaching set}. We use $(x,y)$ to denote a labeled example where $x\in\X$ and $y\in\{-1,1\}$. We denote by $\Hyp$ %\subseteq\{-1,1\}^{\X}$
a finite class of hypotheses. Each element of $\Hyp$ is a function $h:\X\mapsto\mathbb{R}$. The label assigned to $x$ by hypothesis $h$ is $\sgn(h(x))$. The magnitude $|h(x)|$ indicates the confidence hypothesis $h$ has in the label of $x$. For now, let us assume that $\X$ and $\Hyp$ are known to both the teacher and the learner.
In our image classification example, each image may be given by a set of features $x$, and each hypothesis $h(x)=w_h^Tx$ could be a linear function.
%, inferred, for example, %using the method of \citet{2010-nips_crowd-wisdom} on the teaching set,
%as described in Section~\ref{sec:overview} \andreas{update}).
In Section~\ref{sec:experiments_setup}, we discuss the concrete hypothesis spaces used in our crowdsourcing tasks, and how we can elicit them from the crowd.
% workers' labels alone (i.e., without having to explicitly compute any features).
%\adish{TODO: we should clarify here about how we obtain H/X from the crowd itself}.

The teacher has access to the labels $y(x)$ of all the examples $x$ in $\X$. We consider the {\em realizable} setting where $\Hyp$ contains a hypothesis $h^*$ (known to the teacher, but not the learner)
% and thus knows the correct hypothesis  $h^*\in\Hyp$
for which $\sgn(h^*(x))=y(x)$ for all $x\in\X$.
%We will lift this assumption later.
The goal of the teacher is to teach the correct hypothesis $h^*$ to the learner.
The basic assumption behind our approach is that if we can teach the workers to classify $\X$ correctly, then they will be able to generalize to new examples drawn from the same distribution as $\X$ (for which we neither have ground truth labels nor features).
We will verify this assumption experimentally in Section~\ref{sec:experiments_results}.
In the following, we review existing approaches to teaching classifiers, and then present our novel teaching method.

%
%%%%%%%%%%%%%%%%%%%%%%%%%%%%%%%%%%%%%%%%%%%%%%%%%%%%%%%%%%%%%%%%%%%%%%%%%%
%%%%%%%%%%%%%%%%%%%%%%%%%%%%%%%%%%%%%%%%%%%%%%%%%%%%%%%%%%%%%%%%%%%%%%%%%%
\vspace{-3mm}
\subsection{Existing teaching models}\label{sec:related}
\vspace{-1mm}
%\andreas{Not sure how much we need to talk about interactive teaching, since we don't do it?}
%Here, we briefly review existing  work on algorithmic teaching of classifiers.
%A first broad separation can be made into {\em non-interactive} models, where a fixed set of teaching examples is chosen ahead of time, and {\em interactive} models, where examples are chosen adaptively based on the learner's response. A second
In existing methods, a broad separation can be made about assumptions that learners use to process training examples. {\em Noise-free} models assume learners immediately discard hypotheses  inconsistent with observed examples.  As our experiments in Section~\ref{sec:experiments_results} show, such models can be  brittle in practice. In contrast, {\em noise-tolerant} models make less strict assumptions on how workers treat inconsistent hypotheses. %In this work, we focus on the {\em noise-tolerant} setting, which is a natural model for many crowdsourcing tasks.

\vspace{-4mm}
\paragraph{Noise-free teaching:}
In their seminal work, \citet{1992-jcss_kearns_complexity-of-teaching} consider the {\em non-interactive} model: The teacher reveals a sequence of labeled examples, and the learner discards any inconsistent hypotheses (i.e., for which $h(x)\neq y$ for any example $(x,y)$ shown). For a given hypothesis class, the {\em Teaching Dimension} is the smallest number of examples required to ensure that all inconsistent hypotheses are eliminated. %\citet{1992-jcss_kearns_complexity-of-teaching} relate this notion of complexity to other notions in learning theory such as the VC dimension.
More recent work \cite{2009-lata_bulach_developments-in-teaching,2010-jmlr_models-cooperative-teaching-learning, 2010-alt_recursive-teaching-dimension, DuL11} consider models of {\em interactive} teaching, where the teacher, after showing each example, obtains feedback about the hypothesis that the learner is currently implementing. Such feedback can be used to select future teaching examples in a more informed way.  While theoretically intriguing, in this paper we focus on {\em non-interactive} models, which are typically easier to deploy in practice.
%Alternative notions of teaching complexity are introduced \cite{2010-jmlr_models-cooperative-teaching-learning,2010-alt_recursive-teaching-dimension}, which capture ``rational'' learners that reason about the teacher's intent.
%\citet{2009-lata_bulach_developments-in-teaching} consider a setting where, after the teacher shows each labeled example $(x,y)$, the learner reveals his current hypothesis $h$. In case $h(x)\neq y$, the learner adopts an alternative hypothesis $h'$. \citet{2009-lata_bulach_developments-in-teaching} prove that under some assumptions on the learner's transition model, the teaching complexity (i.e., number of labels required) can be reduced. One example is the ``lazy learner model'' of \citet{2009-lata_bulach_developments-in-teaching}, in which hypotheses are endowed by a metric, and the learner is assumed to deterministically move to the closest consistent hypothesis.

\vspace{-4mm}
\paragraph{Noise-tolerant teaching:}
In contrast to the noise-free setting, the practically extremely important {\em noise-tolerant} setting is theoretically much less understood. Very recently, \citet{zhu13} investigates the optimization problem of generating a set of teaching examples that trades off between the expected future error of the learner and the ``effort'' (\ie, number of examples) taken by the teacher, in the special case when the prior of the learner falls into the exponential family, and the learner performs Bayesian inference. Their algorithmic approach does not apply to the problem addressed in this paper. Further, the approach is based on heuristically rounding the solution of a convex program, with no bounds on the integrality gap.

%While these approaches provide intriguing insights on the complexity of teaching, in practical settings they are limited. First, the learner's response is considered to be noise free. Second, the assumptions on the learner's transition model (capturing learning progress) are strong (e.g., the ``lazy learner'' model is deterministic in nature). Lastly, the assumption that the learner's hypothesis is fully revealed is hard to achieve in practice. In this paper, we introduce a Bayesian model relaxing these assumptions.
%%%%%%%%%%%%%%%%%%%%%%%%%%%%%%%%%%%%%%%%%%%%%%%%%%%%%%%%%%%%%%%%%%%%%%%%%%
%%%%%%%%%%%%%%%%%%%%%%%%%%%%%%%%%%%%%%%%%%%%%%%%%%%%%%%%%%%%%%%%%%%%%%%%%%
%\vspace{-2mm}
%\subsection{Recent work on teaching}
\looseness -1 \citet{2013-aaai_teaching} study a similar problem of teaching workers to classify images. The authors empirically investigate a variety of heuristic teaching policies on a set of human subjects for a synthetically generated data set. \Citet{lindseyetal13} propose a method for evaluating and optimizing over parametrized policies with different orderings of positive and negative examples. %The main goal of their paper is to approximate the performance function over the whole policy space without experimenting with every possible policy.
%\andreas{need to cite the Optimizing Instructional Policies NIPS 2013 paper}\gabor{done}
%We tackle the same problem by designing a principled Bayesian teaching framework, inspired by the models of teaching and learning.
None of these approaches offer theoretical performance guarantees of the kind provided in this paper.
\vspace{-2mm}
\section{Model of the Learner}\label{sec:learner}
\vspace{-1mm}
We now introduce our model of the learner,
%We assume that the learner has an initial hypothesis $h_0\in\Hyp$ before the teaching process begins.
%We model the learner in the following way.
%In every round, he is presented with an unlabeled example $x_t$ (e.g., an image), and ``guesses'' a label $l_t$ probabilistically based on his current hypothesis $h_t$.
by %In particular, we
formalizing our assumptions about how she adapts her hypothesis based on the training examples she receives from the teacher. Generally, we assume that the learner is not aware that she is being taught.
%(which is known as {\em collusion} in the literature \cite{2010-jmlr_models-cooperative-teaching-learning})\gabor{I think collusion means something else. It's the same as ``coding trick'', where the teacher and learner has some (secret) agreement on what teacher behavior means what. All the teaching literature wants collusion-free teaching, but that still doesn't mean the learner is not aware he is being taught. Especially for the above cooperative teaching paper. We should remove the word ``collusion''.}.
We assume that she carries out a random walk in the hypothesis space $\Hyp$: She starts at some hypothesis, stays there as long as the training examples received are consistent with it, and randomly jumps to an alternative hypothesis upon an observed inconsistency. Hereby, preference will be given to hypotheses that better agree with the received training.
%chooses the hypothesis $h_t$ that will be used to guess the label of example $x_t$ in time step $t$.
%
%``guesses'' the label $l_t$ when given unlabeled example $x_t$, using his hypothesis $h_t$. We also formalize our probabilistic transition model.
%He then receives the true label $y_t$ from the teacher. Based on this label, he decides to either stick with his current hypothesis ($h_{t+1}=h_t$), or to switch to a different one. In the latter case, the new hypothesis of the learner will be probabilistically chosen based on the history of labeled examples seen, and a distance function between hypotheses.

More formally, we model the learner via a stochastic process, in particular a (non-stationary) Markov chain. Before the first example, the learner randomly chooses a hypothesis $h_1$, drawn from a prior distribution $P_0$. Then, in every round $t$ there are two possibilities:
 If the example $(x_t,y_t)$ received agrees with the label implied by the learner's current hypothesis (i.e., $\sgn(h_t(x_t))=y_t$), she sticks to it: $h_{t+1}=h_t$. On the other hand, if the label $y_t$ disagrees with the learner's prediction $\sgn(h_t(x_t))$, she draws a new hypothesis $h_{t+1}$ based on a distribution $P_t$ constructed in a way that reduces the probability of hypotheses that disagreed with the true labels in the previous steps:
\begin{align}
%  P_{t}(h)=\frac{1}{Z_t}P_0(h)\prod_{s=1}^tP(y_s|h,x_s)^{\mathbb{I}\{\sgn(h(x_s))\neq y_s\}}\,,
  P_{t}(h)=\frac{1}{Z_t}P_0(h)\prod_{\substack{s=1\\y_s\neq\sgn(h(x_s))}}^t P(y_s|h,x_s) \label{eq:jumpprob}%^{\mathbb{I}\{\sgn(h(x_s))\neq y_s\}}\,,\label{eq:jumpprob}
%
%  _{\substack{x\in A\\y(x)\neq\sgn(h(x))}}
%
\end{align}
with normalization factor $$Z_t=\sum_{h\in\Hyp}P_0(h)\prod_{\substack{s=1\\y_s\neq\sgn(h(x_s))}}^tP(y_s|h,x_s).%^{\mathbb{I}\{\sgn(h(x_s))\neq y_s\}}.
$$

In Equation~\eqref{eq:jumpprob}, for some $\alpha>0$, the term $$P(y_s|h,x_s)=\frac{1}{1+\exp(-\alpha h(x_s) y_s)}$$
models a likelihood function, encoding the confidence that hypothesis $h$ places in example $x_s$. Thus, if the example $(x_s,y_s)$ is ``strongly inconsistent'' with $h$ (i.e., $h(x_s)y_s$ takes a large negative value and consequently $P(y_s\mid h,x_s)$ is very small), then the learner will be very unlikely to jump to hypothesis $h$.
%The indicator function in the exponent is merely ensuring that only those hypotheses that disagree with the latest label get weight decrease.
The scaling parameter $\alpha$ allows to control the effect of observing inconsistent examples. The limit $\alpha\to\infty$ results in a behavior where inconsistent hypotheses are completely removed from consideration. This case precisely coincides with the {\em noise-free} learner models classically considered in the literature \cite{1992-jcss_kearns_complexity-of-teaching}.

%Without the exponent, the update rule would be identical to the standard Bayesian update.
%The above distribution $P_t$ can also be calculated recursively in every time step:
%\begin{align*}
%    P_{t}(h)=\frac{1}{Z_t'}P_{t-1}(h)P(y_{t}|h,x_{t})^{\mathbb{I}\{\sgn(h(x_{t}))\neq y_{t}\}}\,.
%\end{align*}
It can be shown (see Lemma~$1$ in the supplementary material), that the marginal probability that the learner implements some hypothesis $h$ in step $t$ is equal to $P_t(h)$, even when the true label and the predicted label agreed in the previous step.
\vspace{-3mm}
\section{Teaching Algorithm}\label{sec:teacher}
\vspace{-1mm}
%\MEMO{
%Model and sampling algorithm inspired from teaching in the crowdsourcing system.}
%%%%%%%%%%%%%%%%%%%%%%%%%%%%%%%%%%%%%%%%%%%%%%%%%%%%%%%%%%%%%%%%%%%%%%%%%%
%In this section we outline the algorithms that we use in our model to select the examples to show to the learner.
Given the learner's prior over the hypotheses $P_0(h)$, how should the teacher choose examples to help the learner narrow down her belief to accurate hypotheses?  By carefully showing examples, the teacher can control the learner's progress by steering her posterior towards $h^*$.

With a slight abuse of notation, if the teacher showed the set of examples $A=\{x_1,\dots,x_t\}$ we denote the posterior distribution by $P_t(\cdot)$ and $P(\cdot|A)$ interchangeably. We use the latter notation when we want to emphasize that the examples shown are the elements of~$A$. With the new notation, we can write the learner's posterior after showing $A$ as
\begin{align*}
  P(h|A) &= \frac{1}{Z(A)}P_0(h)\prod_{\substack{x\in A\\y(x)\neq\sgn(h(x))}}P(y(x)|h,x)\,.
\end{align*}
The ultimate goal of the teacher is to steer the learner towards a distribution with which she makes few mistakes. The expected error-rate of the learner after seeing examples $A=\{x_1,\ldots,x_t\}$ together with their labels $y_i=\sgn(h^*(x_i))$ can be expressed as
\begin{align*}
    \mathbb{E}[\lerr\mid A] &=\sum_{h\in \Hyp} P(h|A)\err(h,h^*)\,, \text{where}\\
    \err(h,h^*)&=\frac{|\{x\in \X:\sgn(h(x))\neq \sgn(h^*(x))\}|}{|\X|}
\end{align*}
is the fraction of examples $x$ from the teaching set $\X$ on which $h$ and $h^*$ disagree about the label. We use the notation $\mathbb{E}[\lerr]=\mathbb{E}[\lerr\mid \{\}]$ as shorthand to refer to the learner's error before receiving training.

Given an allowed tolerance $\epsilon$ for the learner's error, a natural objective for the teacher is to find the smallest set of examples $A^*$ achieving this error, i.e.:
%, ideally as small as possible, in order to guarantee that
\begin{align}A_{\varepsilon}^*=\arg\min_{A\subseteq \X} |A|\text{ s.t. }\mathbb{E}[\lerr\mid A] \leq \epsilon\label{eq:probstat}.\end{align}
We will use the notation $\OPT(\epsilon)=|A^*_\epsilon|$ to refer to the size of the optimal solution achieving error $\epsilon$. Unfortunately, Problem~\eqref{eq:probstat} is a difficult combinatorial optimization problem. The following proposition, proved in the supplement, establishes hardness via a reduction from set cover.
\vspace{1mm}
\begin{proposition}Problem~\eqref{eq:probstat} is NP-hard.\label{prop:hardness}\end{proposition}
%In fact, the special case of $\alpha=\infty$ can be seen to be NP-hard, by reduction from the set cover problem \andreas{add proof?}.
\vspace{-1mm}
Given this hardness, in the following, we introduce an efficient approximation algorithm for Problem~\eqref{eq:probstat}.

The first observation is that, in order to solve Problem~\eqref{eq:probstat}, we can look at the
objective function
\begin{align*}
  R(A) &= \mathbb{E}[\lerr]-\mathbb{E}[\lerr\mid A] \\
  &=\sum_{h\in \Hyp}\left(P_0(h)-P(h|A)\right)\err(h,h^*)\,,
\end{align*}
%This objective
quantifying the expected \emph{reduction in error} upon teaching $A$. Solving Problem~\eqref{eq:probstat} is equivalent to finding the smallest set $A$ achieving error reduction $\mathbb{E}[\lerr]-\epsilon$. Thus, if we could, for each $k$, find a set $A$ of size $k$ maximizing $R(A)$, we could solve Problem~\eqref{eq:probstat}, contradicting the hardness.

The key idea is to replace the objective $R(A)$ with the following surrogate function:
\vspace{-2mm}
\begin{align*}
  F(A) &=\sum_{h\in \Hyp}\left(Q(h)-Q(h|A)\right)\err(h,h^*)\,,\text{where}\\
  Q(h|A) &= P_0(h)\prod_{\substack{x\in A\\y(x)\neq\sgn(h(x))}} P(y(x)|h,x)
\end{align*}
is the \emph{unnormalized posterior} of the learner. As shown in the supplementary material, this surrogate objective function satisfies {\em submodularity}, a natural diminishing returns condition. Submodular functions can be effectively optimized using a greedy algorithm, which, at every iteration, adds the example that maximally increases the surrogate function $F$ \citep{1978-_nemhauser_submodular-max}.
%This insight suggests a natural strategy for optimization:
% the greedy algorithm gives a good approximation for optimizing.
We will show that maximizing $F(A)$ gives us good results in terms of the original, normalized objective function $R(A)$, that is, the expected error reduction of the learner. In fact, we show that running the algorithm until $F(A)\geq \mathbb{E}[\lerr]-P_0(h^*)\epsilon$ is sufficient to produce a feasible solution to Problem~\eqref{eq:probstat}, providing a natural stopping condition. %Hereby, $F(\X)=\mathbb{E}[\lerr]
We call the greedy algorithm for $F(A)$ \algo, and describe it in Policy~\ref{alg:UGTP}.
\begin{algorithm}[tb]
   \caption{Teaching Policy \algo}\label{alg:UGTP}
\begin{algorithmic}[1]
   \STATE{{\bfseries Input:} \looseness -1 examples $\X$, hyp.~$\Hyp$, prior $\prior$, error $\epsilon$.}
   \STATE{{\bfseries Output:} teaching set $A$}
   \STATE{$A\gets\emptyset$}
   \WHILE{$F(A)<\mathbb{E}[\lerr]-P_0(h^*)\epsilon$}
   \STATE{$x\gets\argmax_{x\in\X}\left(F(A\cup\{x\})\right)$}
   \STATE $A\gets A\cup\{x\}$
   %\IF {there exists two neighboring polytops $\poly$ and $\poly'$ s.t. $\sum_h \pt{h}{i}h(\poly)>0$ and  $\sum_h \pt{i}{h} h(\poly')<0$}
%   \STATE select $x_i$ uniformly at random from $\poly$ or $\poly'$.
%   \ELSE
%   \STATE select $x_i$ from polytop $\poly=\argmin_{\poly\in \polytop} |\sum_h\pt{h}{i} h(\poly)|$
%   \ENDIF
%   \STATE $\forall h\in\Hyp$ update $\pt{h}{i+1}$ according to \eqref{teacher_update} and $i\rightarrow i+1$.
   \ENDWHILE
\end{algorithmic}
\end{algorithm}
%We use the random teacher as one of our baseline algorithms. As the name suggests, the random teacher selects the next example to show randomly based on some (usually the uniform) distribution over examples that have not been shown yet.
%\adish{We need to clarify about whether we assume the knowledge of $\alpha$, $\beta$. Perhaps, a forward reference Fig.~\ref{fig:experiment-set-VW-ModelMismatch} that our teaching model is robust even when exact values of the worker are not known?}\andreas{Done in Sec 3}

Note that in the limit $\alpha\rightarrow \infty$, $F(A)$ quantifies the prior mass of all hypotheses $h$ (weighted by $\err(h,h^*)$) that are inconsistent with the examples $A$. Thus, in this case, $F(A)$ is simply a weighted coverage function, consistent with classical work in {\em noise-free} teaching \cite{1992-jcss_kearns_complexity-of-teaching}.

\vspace{-3mm}
\subsection{Approximation Guarantees}\label{sec:anal}
\vspace{-1mm}
The following theorem ensures that if we choose the examples in a greedy manner to maximize our surrogate objective function $F(A)$, as done by Policy~\ref{alg:UGTP}, we are close to being optimal in some sense.
\begin{theorem}\label{thm:UGTP}
  Fix $\epsilon>0$. The \algo~Policy~\ref{alg:UGTP} terminates after at most $\OPT(P_0(h^*)\epsilon/2)\log\frac{1}{P_0(h^*)\epsilon}$ steps with a set $A$ such that $\mathbb{E}[\lerr\mid A]\leq \epsilon$.
  %, where $\OPT(x)$ is the minimum number of examples $A$ needed to achieve $\mathbb{E}[\lerr\mid A]\leq x$.
\end{theorem}
\vspace{-1mm}
Thus, informally, Policy~\ref{alg:UGTP} uses a near-minimal number of examples when compared to {\em any} policy achieving $O(\epsilon)$ error (viewing $P_0(h^*)$ as a constant).

The main idea behind the proof of this theorem is that we first observe that $F(A)$ is submodular and thus the greedy algorithm gives a set reasonably close to $F$'s optimum. Then we analyze the connection between maximizing $F(A)$ and minimizing the expected error of the learner, $\mathbb{E}[\lerr\mid A]$.
A detailed proof can be found in the supplementary material.

Note that maximizing $F(A)$ is not only sufficient, but also {\em necessary} to achieve $\epsilon$ precision. Indeed, it is immediate that
  $
    P(h|A) \geq Q(h|A)\,,
$
  which in turn leads to
  \vspace{-1mm}
\begin{align*}
  \mathbb{E}[\lerr|A]\!=\!\sum_{h\in \Hyp}\!P(h|A)\err(h,h^*) &\!\geq\!   \sum_{h\in \Hyp}Q(h|A)\err(h,h^*)\\
  &=\mathbb{E}[\lerr]-F(A)\,.
\end{align*}
Thus, if $\mathbb{E}[\lerr]-F(A)>\epsilon$, then the expected posterior error $\mathbb{E}[\lerr\mid A]$ of the  learner is also greater than $\epsilon$.

%%%%%%%%%%%%%%%%%%%%%%%%%%%%%%%%%%%%%%%%%%%%%%%%%%
\begin{figure*}[t]
\centering
\includegraphics[width=0.85\textwidth]{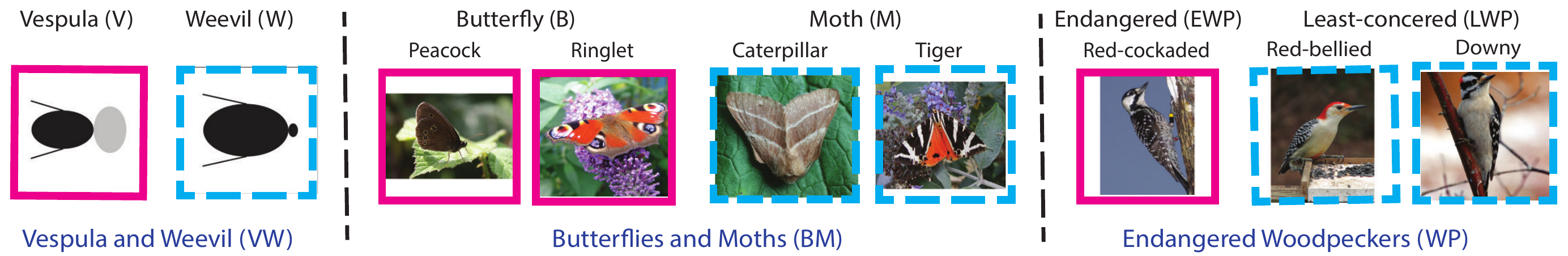}
\vspace{-4mm}
\caption{Sample images of all the three data sets used for the experiments.}
\vspace{-5mm}
\label{fig:fig_X_All}
\end{figure*}
%%%%%%%%%%%%%%%%%%%%%%%%%%%%%%%%%%%%%%%%%%%%%%%%%%

  \vspace{-2mm}
% !TEX root =  teaching.tex
\subsection{Teaching Complexity for Linear Separators}
  \vspace{-1mm}
Theorem~\ref{thm:UGTP} shows that greedily optimizing $F(A)$ leads to low error with a number of examples not far away from the optimal. Now we show that, under some additional assumptions, the optimal number of examples is not too large.

We consider the important case where the set of hypotheses $\Hyp = \SetHyp$ consists of linear separators
$h(x)=w_h^T x+b_h$
for some weight vector $w_h\in\real^d$ and offset $b_h\in\real$. The label predicted by $h$ for example $x$ is $\sgn(w_h^T x+b_h)$.
%where $a_i\in \real^d$, $||a_i||_2=1$ and the offset $b_i\in \real$ is such that $|b_i|\leq b$ for some constant $b$. Then, the label predicted by $h_i$ is $y =\sgn( a_i^Tx+b_i)$.

We introduce an additional assumption, namely $\lambda$-richness of $(\X,\Hyp)$. First notice that $\Hyp$ partitions $\X$ into polytopes (intersections of half-spaces), where within one polytope, all  examples are labeled the same by every hypothesis, that is, within a polytope $\poly$, for every $x,x'\in\poly\subseteq\real^d$ and $h\in\Hyp$, $\sgn(h(x))=\sgn(h(x'))$.
We say that $\X$ is $\lambda$-rich if any $\poly$ contains at least $\lambda$ examples. In other words, if the teacher needs to show (up to) $\lambda$ distinct examples to the learner from the same polytope in order to reduce her error below some level, this can be done.

 \begin{theorem}\label{thm:stronger}
  Fix $\varepsilon>0$. Suppose that the hypotheses are hyperplanes in $\real^{d}$ and that $(\X,\Hyp)$ is $(8\log^2\frac{2}{\epsilon})$-rich. Then the \algo policy achieves learner error less than $\epsilon$ after at most
  $
    m = 8\log^2\frac{2}{\epsilon}
  $
  teaching examples. %, for some problem-dependent constant $C$.
\end{theorem}
\vspace{-2mm}
The proof of this theorem is in the supplementary material. In a nutshell, the proof works by establishing the existence (via the probabilistic method) of a teaching policy for which the number of examples needed can be bounded -- hence also bounding the optimal policy -- and then using Theorem~\ref{thm:UGTP}.
\vspace{-2mm}

%%%%%%%%%%%%%%%%%%%%%%%%%%%%%%%%%%%%%%%%%%%%%%%%%%%%%%%%%%%%%%%%%%%%%%%%%%
%%%%%%%%%%%%%%%%%%%%%%%%%%%%%%%%%%%%%%%%%%%%%%%%%%%%%%%%%%%%%%%%%%%%%%%%%%
% !TEX root =  teaching.tex
%%%%%%%%%%%%%%%%%%%%%%%%%%%%%%%%%%%%%%%%%%%%%%%%%%%%%%%%%%%%%%%%%%%%%%%%%%
%%%%%%%%%%%%%%%%%%%%%%%%%%%%%%%%%%%%%%%%%%%%%%%%%%%%%%%%%%%%%%%%%%%%%%%%%%
\section{Experimental Setup}\label{sec:experiments_setup}
\vspace{-1mm}
In our experiments, we consider three different image classification tasks: {i)} classification of synthetic insect images into two hypothetical species {\em Vespula and Weevil} (\emph{VW}); {ii)} distinguishing {\em butterflies and moths} on real images (\emph{BM}); and {iii)} identification of birds belonging to an {\em endangered species of woodpeckers} from real images (\emph{WP}). Our teaching process requires a known feature space for image dataset $\X$ (\emph{i.e.} the teaching set of images) and a hypothesis class $\Hyp$. While $\X$ and $\Hyp$ can be controlled by design for the synthetic images, we illustrate different ways on how to automatically obtain a crowd-based embedding of the data for real images. We now discuss in detail the experimental setup of obtaining $\X$ with its feature space and $\Hyp$ for the three different data sets used in our classification tasks.

%%%%%%%%%%%%%%%%%%%%%%%%%%%%%%%%%%%%%%%%%%%%%%%%%%%%%%%%%%%%%%%%%%%%%%%%%%
\vspace{-2.5mm}
\subsection{Vespula vs.~Weevil}
\vspace{-1mm}
We first generate a classification problem using synthetic images $\X$ in order to allow controlled experimentation. As a crucial advantage, in this setting the hypothesis class $\Hyp$ is known by design, and the task difficulty can be controlled. Furthermore, this setting ensures that workers have no prior knowledge of the image categories.
%As the images are generated by a controlled process, exact euclidean embedding of data $\X$ is known.

\vspace{-0.5mm}
{\bf Dataset $\X$ and feature space:} We generated synthetic images of insects belonging to two hypothetical species: \emph{Weevil} and \emph{Vespula}.  The task is to classify whether a given image contains a Vespula or not. The images were  generated by varying body size and color as well as head size and color. A given image $x_i$ can be distinguished based on the following two-dimensional feature vector $x_{i}=[x_{i,1}=f_1,x_{i,2}=f_2]$ -- i) $f_1$: the head/body size ratio, ii) $f_2$: head/body color contrast. Fig.~\ref{fig:vw_picked_HX} shows the embedding of this data set in a two-dimensional space based on these two features. Fig.~\ref{fig:fig_X_All} shows sample images of the two species and illustrates that Weevils have short heads with color similar to their body, whereas Vespula are distinguished by their big and contrasting heads. A total of 80 images per species were generated by sampling the features $f_1$ and $f_2$ from two bivariate Gaussian distributions: $(\mu=[0.10,0.13]$, $\Sigma=[0.12,0;0,0.12]$) for Vespula and $(\mu=[-0.10,-0.13]$, $\Sigma=[0.12,0;0,0.12])$ for Weevil. A separate test set of 20 images per species were generated as well, for evaluating learning performance. % of both synthetic and MTurk workers.

\vspace{-0.5mm}
{\bf Hypothesis class $\Hyp$:}
Since we know the exact feature space of $\mathcal{X}$, we can use any parametrized class of functions $\mathcal{H}$ on $\X$. In our experiments, we use a class of linear functions for $\mathcal{H}$
%. Instead of considering every possible linear hypothesis, 
, and further restricting $\Hyp$ to eight clusters of hypotheses, centered at the origin and rotated by $\sfrac{\pi}{4}$ from each other. Specifically, we sampled the parameters of the linear hypotheses from the following multivariate Gaussian distribution: $(\mu_i=[\sfrac{\pi}{4} \cdot i, 0]$, $\Sigma_i=[2,0;0,0.005])$, where $i$ varies from $0$ to $7$. Each hypothesis captures a different set of cues about the features that workers could reasonably have: {i)} ignoring a feature, {ii)} using it as a positive signal for Vespula, and {iii)} using it as a negative signal for Vespula. Amongst the generated hypothesis, we picked target hypothesis $h^*$ as the one with minimal error on teaching set $\X$. In order to ensure realizability, we then removed any data points $x \in \X$ where $\sgn(h^*(x))\neq y(x)$. Fig.~\ref{fig:vw_picked_HX} shows a subset of four of these hypothesis, with the target hypothesis $h^*$ represented in red. The prior distribution $P_0$ is chosen as uniform.
% \andreas{true?}
%%%%%%%%%%%%%%%%%%%%%%%%%%%%%%%%%%%%%%%%%%%%%%%%%
\begin{figure*}[t!]
\centering
   \subfigure[$\X$ and $\Hyp$ for VW]{
     \includegraphics[width=0.28\textwidth]{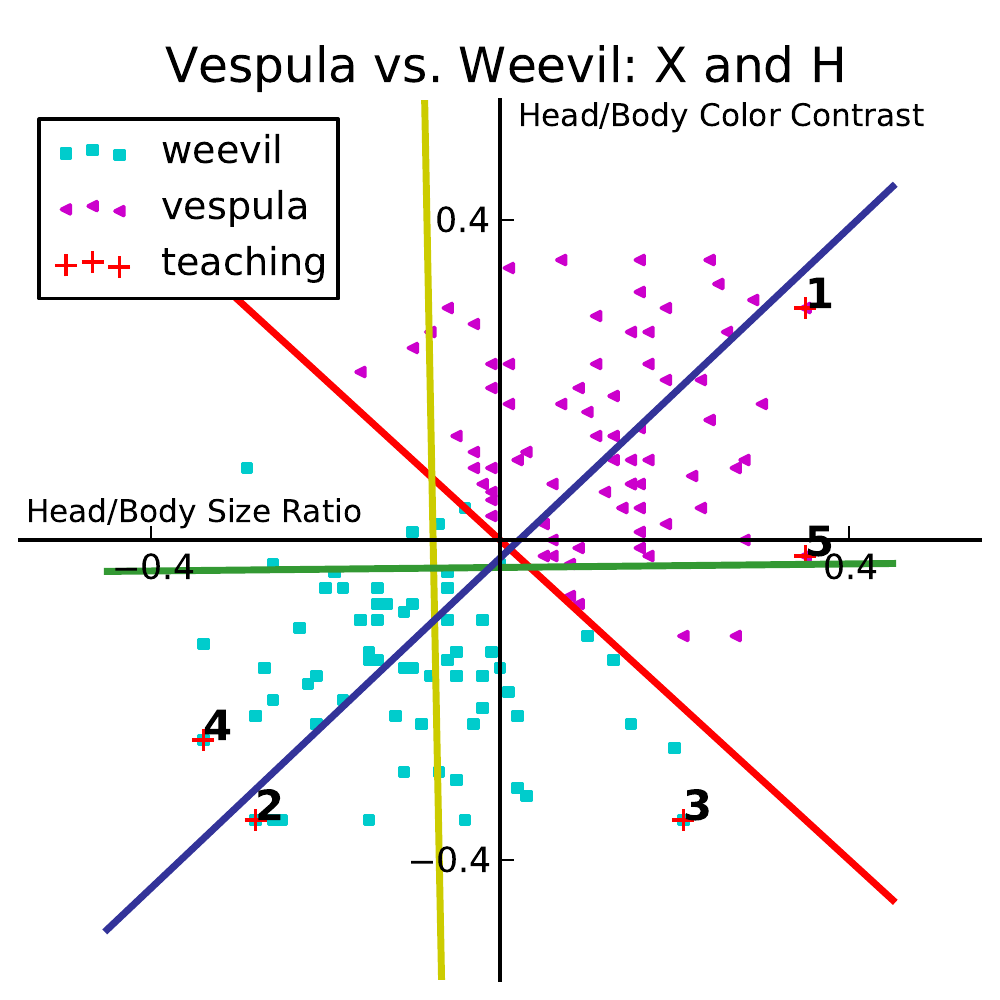}
     \label{fig:vw_picked_HX}
   }
   \subfigure[$\X$ and $\Hyp$ for BM]{
     \includegraphics[width=0.28\textwidth]{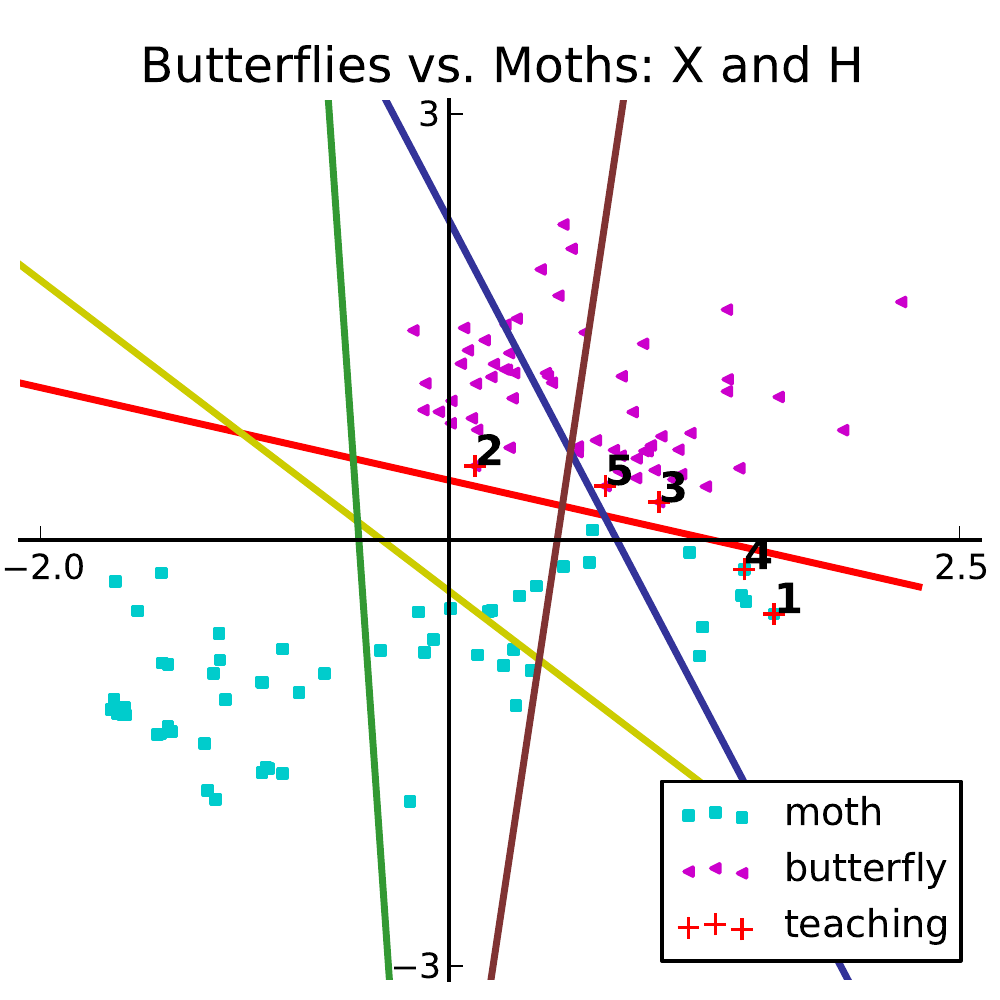}
     \label{fig:bm_picked_HX}
   }
   \subfigure[Features for WP]{
     \includegraphics[width=0.38\textwidth]{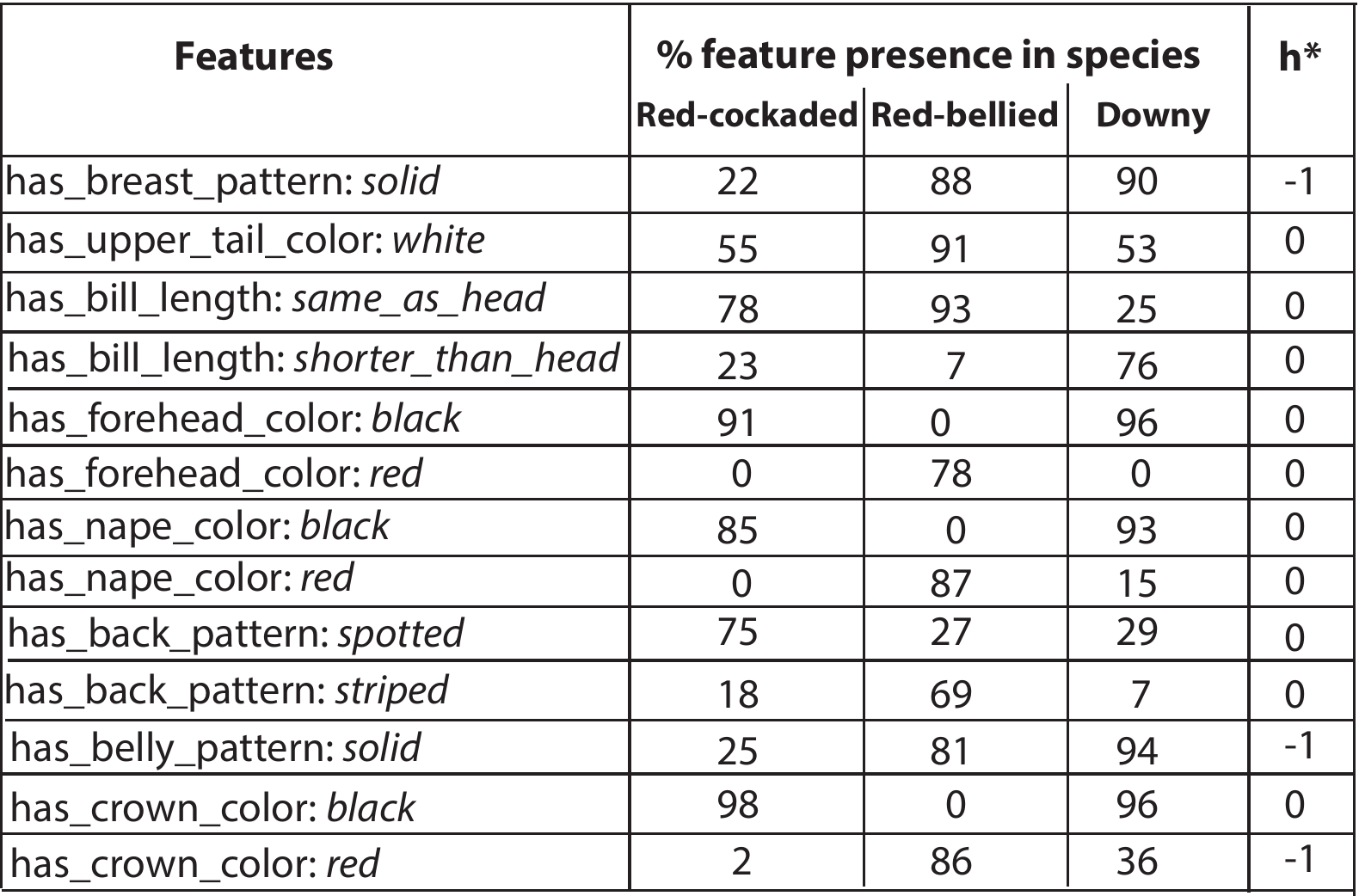}
     \label{fig:wp_picked_HX}
   }
\vspace{-4mm}
\caption{ (a) shows the 2-D embedding of synthetic images for \emph{Weevil} and \emph{Vespula} for the features: head/body size proportion ($f_1$) and head/body color contrast ($f_2$), normalized around origin.  It shows four of the hypotheses in $\Hyp$, with the target hypothesis $h^*$ in red.   (b) shows the 2-D embedding of images for the \emph{Moth} and \emph{Butterfly} data set, and the hypothesis for a small set of workers, as obtained using the approach of \citet{2010-nips_crowd-wisdom}. (c) shows the 13 features used for representation of woodpecker images and the $w_{h^*}$ vector of the target hypotheses. It also lists  the average number of times  a particular feature is present in the images of a given species.
%{\bf iv)} Fig.~\ref{fig:all_picked} illustrates the sample images of three data sets and the sequence of the examples picked by different algorithms for showing to synthetic as well as MTurk workers.
}
\vspace{-4mm}
\label{fig:all_HX}
\end{figure*}
%%%%%%%%%%%%%%%%%%%%%%%%%%%%%%%%%%%%%%%%%%%%%%%%%
%%%%%%%%%%%%%%%%%%%%%%%%%%%%%%%%%%%%%%%%%%%%%%%%%%%%%%%%%%%%%%%%%%%%%%%%%%
%%%%%%%%%%%%%%%%%%%%%%%%%%%%%%%%%%%%%%%%%%%%%%%%%%%%%%%%%%%%%%%%%%%%%%%%%%
\vspace{-3mm}
\subsection{Butterflies vs.~Moths}
\vspace{-1mm}
{\bf Dataset images $\X$:} As our second dataset, we used a collection of 200 real images of four species of butterflies and moths from publicly available images\footnote{Imagenet: \url{http://www.image-net.org/}}: i) {\em Peacock Butterfly}, ii) {\em Ringlet Butterfly}, iii) {\em Caterpillar Moth}, iv) {\em Tiger Moth}, as shown in Fig.~\ref{fig:fig_X_All}. The task is to classify whether a given image contains a butterfly or not. While Peacock Butterfly and Caterpillar Moth are clearly distinguishable as butterflies and moths, Tiger Moth and Ringlet Butterfly are visually hard to classify correctly.  We used 160 of these images (40 per sub-species) as teaching set $\X$ and the remaining 40 (10 per sub-species) for testing.
%, two species belonging to butterflies and two belonging to moths 
%On this dataset, 

\vspace{-0.5mm}
{\bf Crowd-embedding of $\X$:}
A Euclidean embedding of $\mathcal{X}$ for such an image set is not readily available. Human-perceptible features for such real images may be difficult to compute. In fact, this challenge is one major motivation for using crowdsourcing in image annotation. However, several techniques do exist that allow estimating such an embedding from a small set of images and a limited number of crowd labels. In particular, we used the approach of \citet{2010-nips_crowd-wisdom} as a preprocessing step. \citeauthor{2010-nips_crowd-wisdom}~propose a generative Bayesian model for the annotation process of the images by the workers and then use an inference algorithm to jointly estimate a low-dimensional embedding of the data, as well as a collection of linear hypotheses -- one for each annotator -- that best explain their provided labels.

We requested binary labels (of whether the image contains a butterfly) for our teaching set $\X$, $|\X|=160$, from a set of 60 workers.  By using the software \emph{CUBAM}\footnote{CUBAM: \url{https://github.com/welinder/cubam}},  implementing the approach of \citeauthor{2010-nips_crowd-wisdom}, 
%By feeding the crowd labels to CUBAM, 
we inferred a 2-D embedding of the data, as well as linear hypotheses corresponding to each of the 60 workers who provided the labels.  Fig.~\ref{fig:bm_picked_HX} shows this embedding of the data, as well as a small subset of workers' hypotheses as colored lines. %Note that this embedding is {\em only} available on the teaching set, not on the test set.

\vspace{-0.5mm}
{\bf Hypothesis class $\Hyp$:}
The 60 hypothesis obtained through the crowd-embedding provide a prior distribution over linear hypotheses that the workers in the crowd may have been using. Note that these hypotheses capture various idiosyncrasies (termed ``schools of thought'' by \citeauthor{2010-nips_crowd-wisdom}) in the workers' annotation behavior -- i.e., some workers were more likely to classify certain moths as butterflies and vice versa.
To create our hypothesis class $\Hyp$, we randomly sampled 15 hypotheses from these. Additionally, we fitted a linear classifier that best separates the classes and used it as target hypothesis $h^*$, shown in red in Fig.~\ref{fig:bm_picked_HX}. The few examples in $\X$ that disagreed with $h^*$ were removed from our teaching set, to ensure realizability.

\vspace{-0.5mm}
{\bf Teaching the rest of the crowd:}
The teacher then uses this embedding and hypotheses in order to teach the rest of the crowd. 
%We note that the set of workers used for teaching is different from the set of workers that were requested labels for images in $\X$ to obtain the embedding. 
We emphasize that -- crucially -- the embedding is {\em not} required for test images. Neither the workers nor the system used any information about sub-species in the images.
%%%%%%%%%%%%%%%%%%%%%%%%%%%%%%%%%%%%%%%%%%%%%%%%%%%%%%%%%%%%%%%%%%%%%%%%%%
%%%%%%%%%%%%%%%%%%%%%%%%%%%%%%%%%%%%%%%%%%%%%%%%%%%
\begin{figure*}[t!]
\centering
   \subfigure[Data set VW: Test Error]{
     \includegraphics[width=0.31\textwidth]{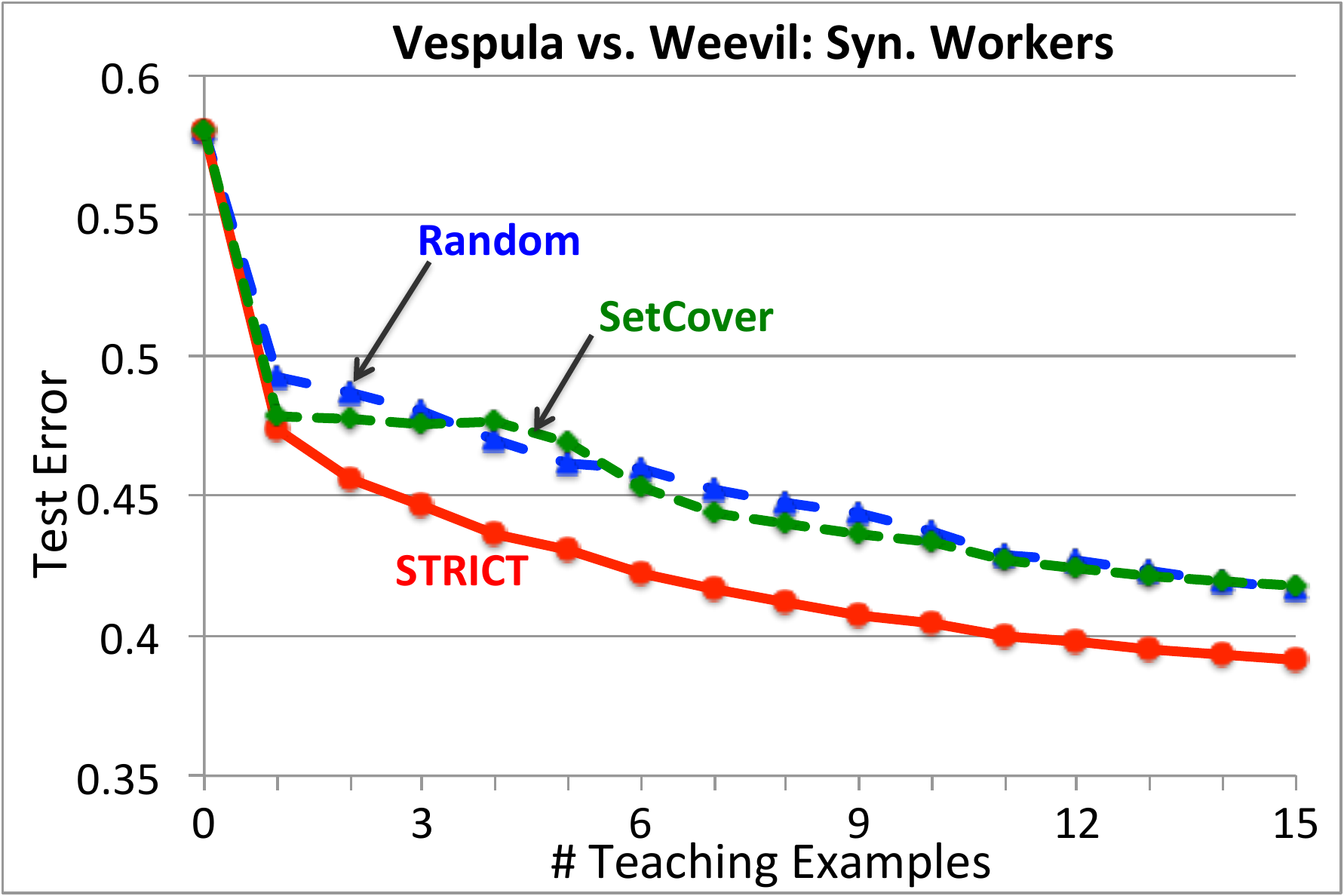}
     \label{fig:experiment-set-VW-Syn-Err}
   }
   \subfigure[Data set VW: Robustness  w.r.t $\alpha$]{
     \includegraphics[width=0.31\textwidth]{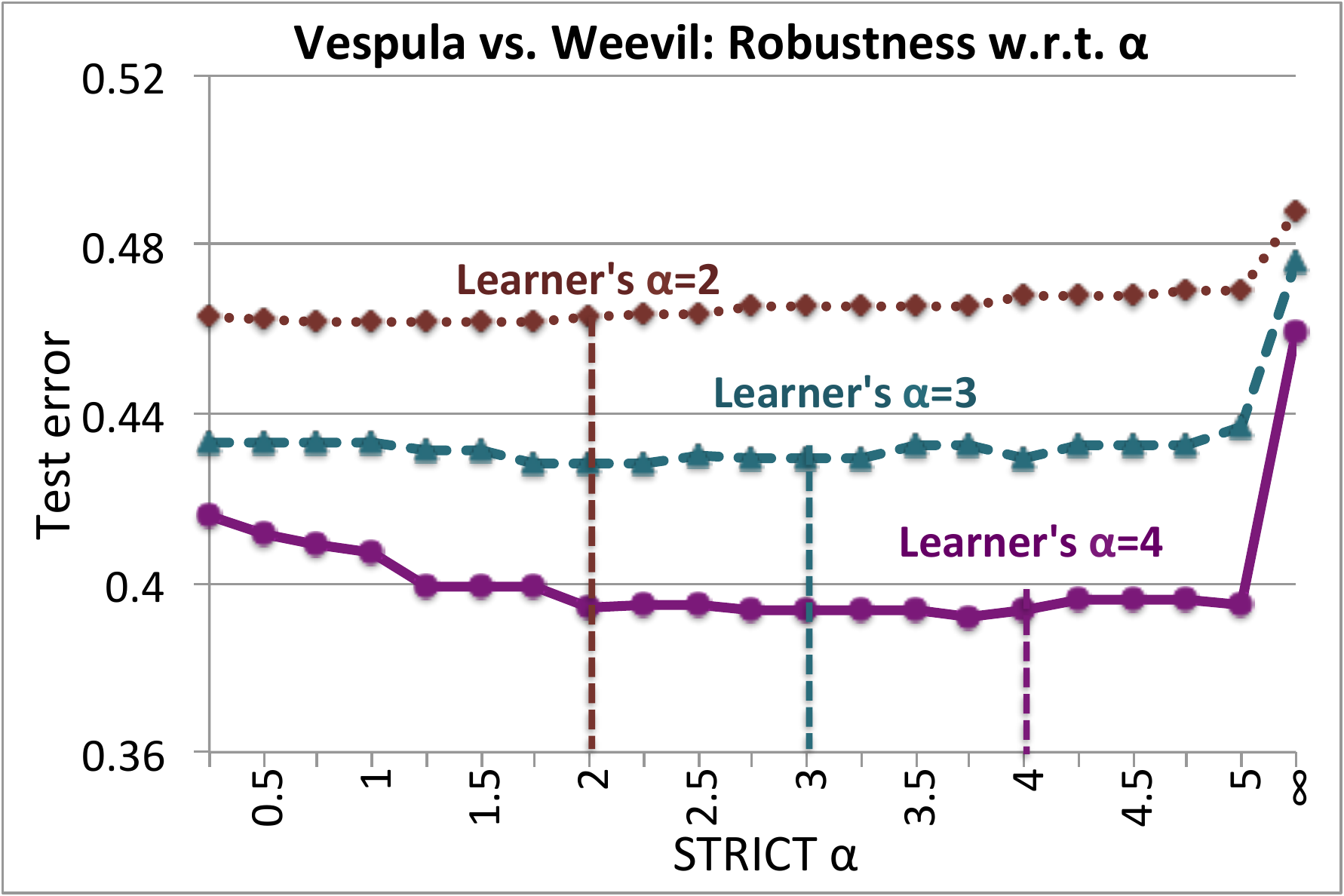}
     \label{fig:experiment-set-VW-Syn-VaryingAlpha}
   }
   \subfigure[Data set VW: Difficulty Level]{
     \includegraphics[width=0.31\textwidth]{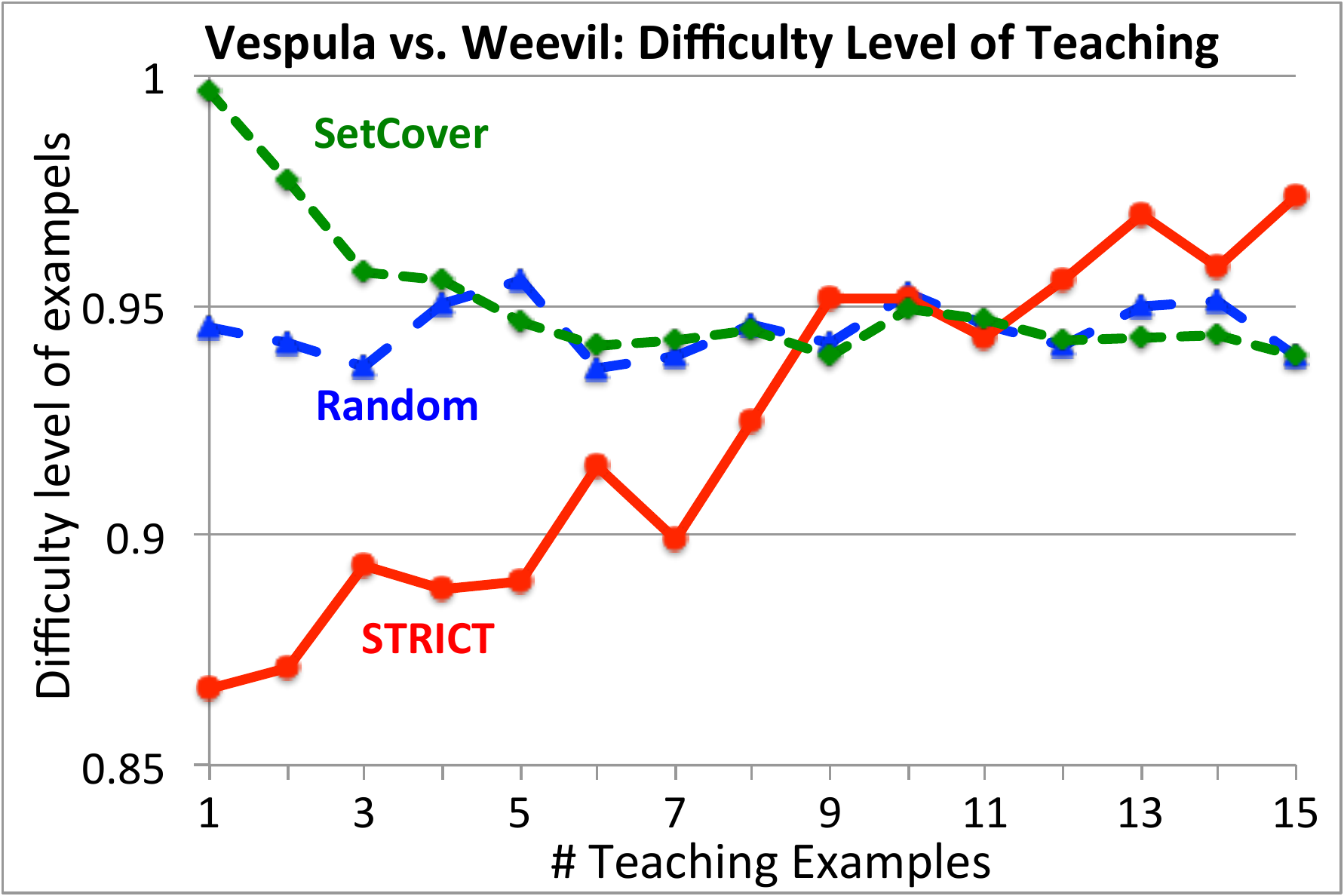}
     \label{fig:experiment-set-VW-Syn-Difficulty}
   }
\vspace{-4mm}
\caption{%Fig.~\ref{fig:experiment-set-VW-Syn-Err} 
(a) compares the algorithms' teaching performance in terms of simulated workers' test error (VW task). %Fig.~\ref{fig:experiment-set-VW-Syn-VaryingAlpha} 
(b) shows the robustness of \algo w.r.t.~unknown $\alpha$ parameters of the learners. Thus, a noise-tolerant teacher (\ie, $\alpha<\infty$) performs much better than noise-free {\em SetCover} teaching, even with misspecified $\alpha$.  
%Fig.~\ref{fig:experiment-set-VW-Syn-Difficulty}  
(c) shows how the difficulty of \algo's examples naturally increase during teaching. 
%picked by our algorithm increases in the teaching process.
}
\vspace{-4mm}
\label{fig:results-syn}
\end{figure*}
%%%%%%%%%%%%%%%%%%%%%%%%%%%%%%%%%%%%%%%%%%%%%%%%%%
%%%%%%%%%%%%%%%%%%%%%%%%%%%%%%%%%%%%%%%%%%%%%%%%%%%%%%%%%%%%%%%%%%%%%%%%%%
%%%%%%%%%%%%%%%%%%%%%%%%%%%%%%%%%%%%%%%%%%%%%%%%%%%%%%%%%%%%%%%%%%%%%%%%%%
%%%%%%%%%%%%%%%%%%%%%%%%%%%%%%%%%%%%%%%%%%%%%%%%%%%%%%%%%%%%%%%%%%%%%%%%%%
\vspace{-3mm}
\subsection{Endangered Woodpecker Bird Species}
\vspace{-1mm}
{\bf Dataset images $\X$:} Our third classification task is inspired from the eBird citizen science project \cite{sullivan2009ebird} and the goal of this task is to identify birds belonging to an endangered species of woodpeckers.  We used a collection of 150 real images belonging to three species of woodpeckers from a publicly available dataset \cite{WahCUB_200_2011}, with one endangered species: i) {\em Red-cockaded} woodpecker and other two species belonging to the least-concerned category: ii) {\em Red-bellied woodpecker}, iii) {\em Downy woodpecker}. On this dataset, the task is to classify whether a given image contains a red-cockaded woodpecker or not. We used 80 of these images (40 per red-cockaded, and 20 each per the other two species of least-concerned category) for teaching (\ie, dataset $\X$). We also created a testing set of 20 images (10 for red-cockaded, and 5 each for the other two species).
%\footnote{CUB-200-2011: \url{http://www.vision.caltech.edu/visipedia/CUB-200-2011.html}}
%\andreas{misspelling. Should be Red-cockaded}

\vspace{-1mm}
{\bf Crowd-embedding of $\X$:}
%As in the case of the Butterflies and Moths dataset, 
We need to infer an embedding and hypothesis space of the teaching set for our teaching process. While an approach similar to the one used for the BM task is applicable here as well, we considered an alternate option of using metadata associated with these images, elicited from the crowd, as further explained below.

\vspace{-1mm}
Each image in this dataset is annotated with 312 binary attributes, for example, \emph{has\_forehead\_color:black}, or \emph{has\_bill\_length:same\_as\_head}, through workers on MTurk. The features can take values \{+1, -1, 0\} indicating the presence or absence of an attribute, or uncertainty (when the annotator is not sure or the answer cannot be inferred from the image given). Hence, this gives us an embedding of the data in $\mathbb{R}^{312}$. To further reduce the dimensionality of the feature space, we pruned the features which are not informative enough for the woodpecker species. We considered all the species of woodpeckers present in the dataset (total of 6), simply computed the average number of times a given species is associated positively with a feature, and then looked for features with maximal variance among the various species. By applying a simple cutoff of 60 on the variance, we picked the top $d=13$ features as shown in Fig~\ref{fig:wp_picked_HX}, also listing the average number of times the feature is associated positively with the three species.
%\andreas{how many}

\vspace{-0.5mm}
{\bf Hypothesis class $\Hyp$:}
%We used a simple approach to generate the hypothesis class, given the feature space obtained above.
We considered a simple set of linear hypotheses $h(x)=w^T x$ for $w\in\{+1,0,-1\}^d$, which place a weight of \{+1, 0, -1\} on any given feature and passing through the origin. The intuition behind these simple hypotheses is to capture the cues that workers could possibly use or learn for different features: ignoring a feature ($0$), using it as a positive signal ($+1$), and  using it as a negative signal ($-1$). Another set of simple hypotheses that we explored  are conjunctions and disjunctions of these features that can be created by setting the appropriate offset factor $b_h$ \cite{1992-colt_exact-specification-by-examples}. Assuming that workers focus only on a small set of features, we considered sparse hypotheses with non-zero weight on only a small set of features. To obtain the target hypothesis, we enumerated all possible hypotheses that have non-zero weight for at most three features. We then picked as $h^*$ the hypothesis with minimal error on  $\X$ (shown in Fig~\ref{fig:wp_picked_HX}). Again, we pruned the few examples in $\X$ which disagreed with $h^*$ to ensure realizability.
As hypothesis class $\Hyp$, we considered all hypotheses with a non-zero weight for at most two features along with the target $h^*$, resulting in a hypothesis class of size 339.

\vspace{-0.5mm}
{\bf Teaching the rest of the crowd:}
Given this embedding and hypothesis class, the teacher then uses the same approach as two previous datasets to teach the rest of the crowd. Importantly, this embedding is {\em not} required for test images.
%%%%%%%%%%%%%%%%%%%%%%%%%%%%%%%%%%%%%%%%%%%%%%%%%%%%%%%%%%%%%%%%%%%%%%%%%%
%%%%%%%%%%%%%%%%%%%%%%%%%%%%%%%%%%%%%%%%%%%%%%%%%%%%%%%%%%%%%%%%%%%%%%%%%%

% !TEX root =  teaching.tex
%%%%%%%%%%%%%%%%%%%%%%%%%%%%%%%%%%%%%%%%%%%%%%%%%%%%%%%%%%%%%%%%%%%%%%%%%%
%%%%%%%%%%%%%%%%%%%%%%%%%%%%%%%%%%%%%%%%%%%%%%%%%%%%%%%%%%%%%%%%%%%%%%%%%%
\vspace{-3mm}
\section{Experimental Results}\label{sec:experiments_results}
\vspace{-1mm}
Now we present our experimental results, consisting of simulations and actual annotation tasks on MTurk.
%on the performance of our teaching algorithm, \algo, on the above datasets using a set of simulated learners, as well as human learners (workers on MTurk) using an actual annotation task on the MTurk platform.

\vspace{-0.5mm}
{\bf Metrics and baselines:} Our primary  performance metric is the test error (avg.~classification error of the learners), of simulated or MTurk workers on a hold-out test data set. %, as we vary the length of teaching phase.
We compare \algo against two baseline teachers: {\em Random} (picking uniformly random examples), and  {\em SetCover} (the classical noise-free teaching model introduced in Section~\ref{sec:learner}).
%%%%%%%%%%%%%%%%%%%%%%%%%%%%%%%%%%%%%%%%%%%%%%%%%%%%%%%%%%%%%%%%%%%%%%%%%%
%%%%%%%%%%%%%%%%%%%%%%%%%%%%%%%%%%%%%%%%%%%%%%%%%%
\begin{figure*}[t!]
\centering
   \subfigure[Examples in teaching sequence]{
     \includegraphics[width=0.80\textwidth]{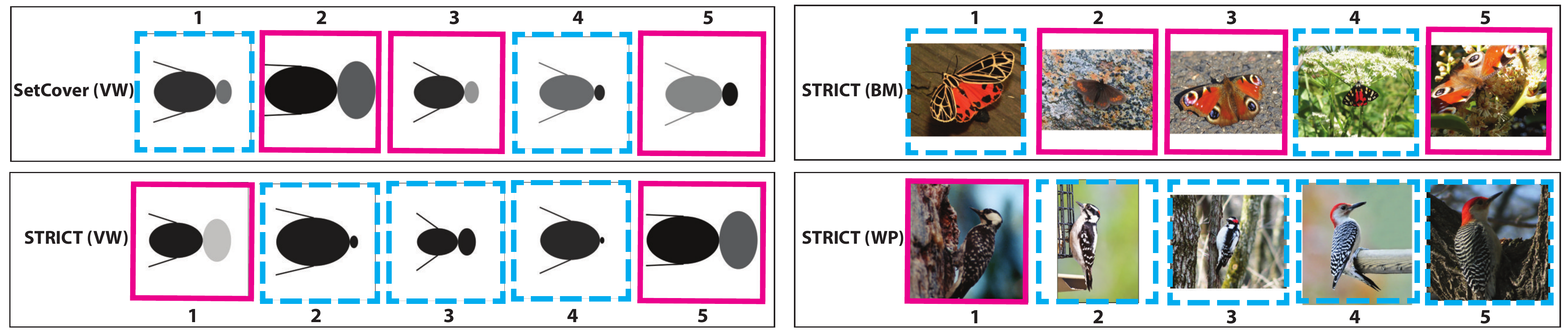}
     \label{fig:all_picked}
   }\\[-3mm]
   \subfigure[Test error on VW dataset]{
     \includegraphics[width=0.31\textwidth]{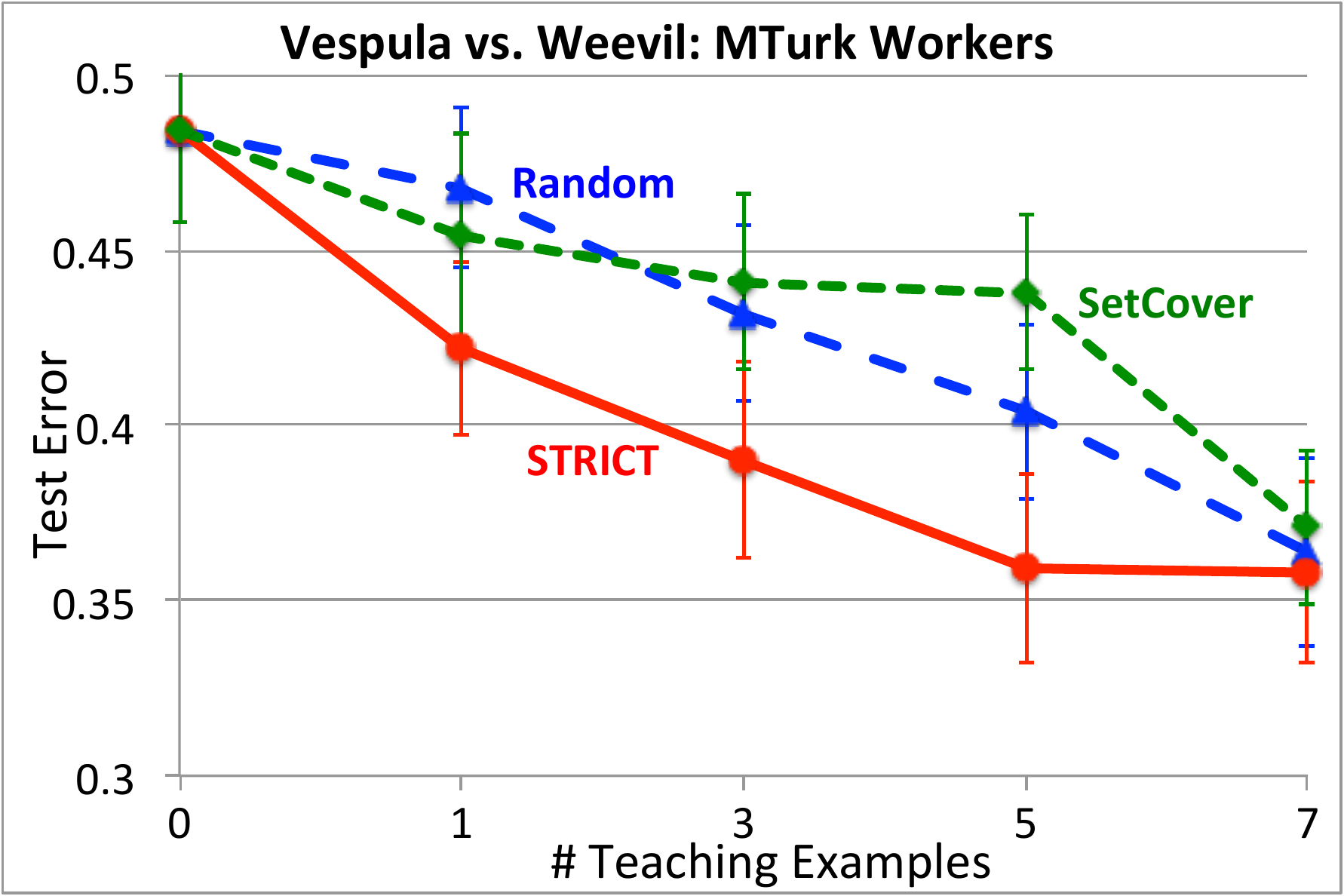}
     \label{fig:experiment-set-VW-MTurk}
   }
   \subfigure[Test error on BM dataset]{
     \includegraphics[width=0.31\textwidth]{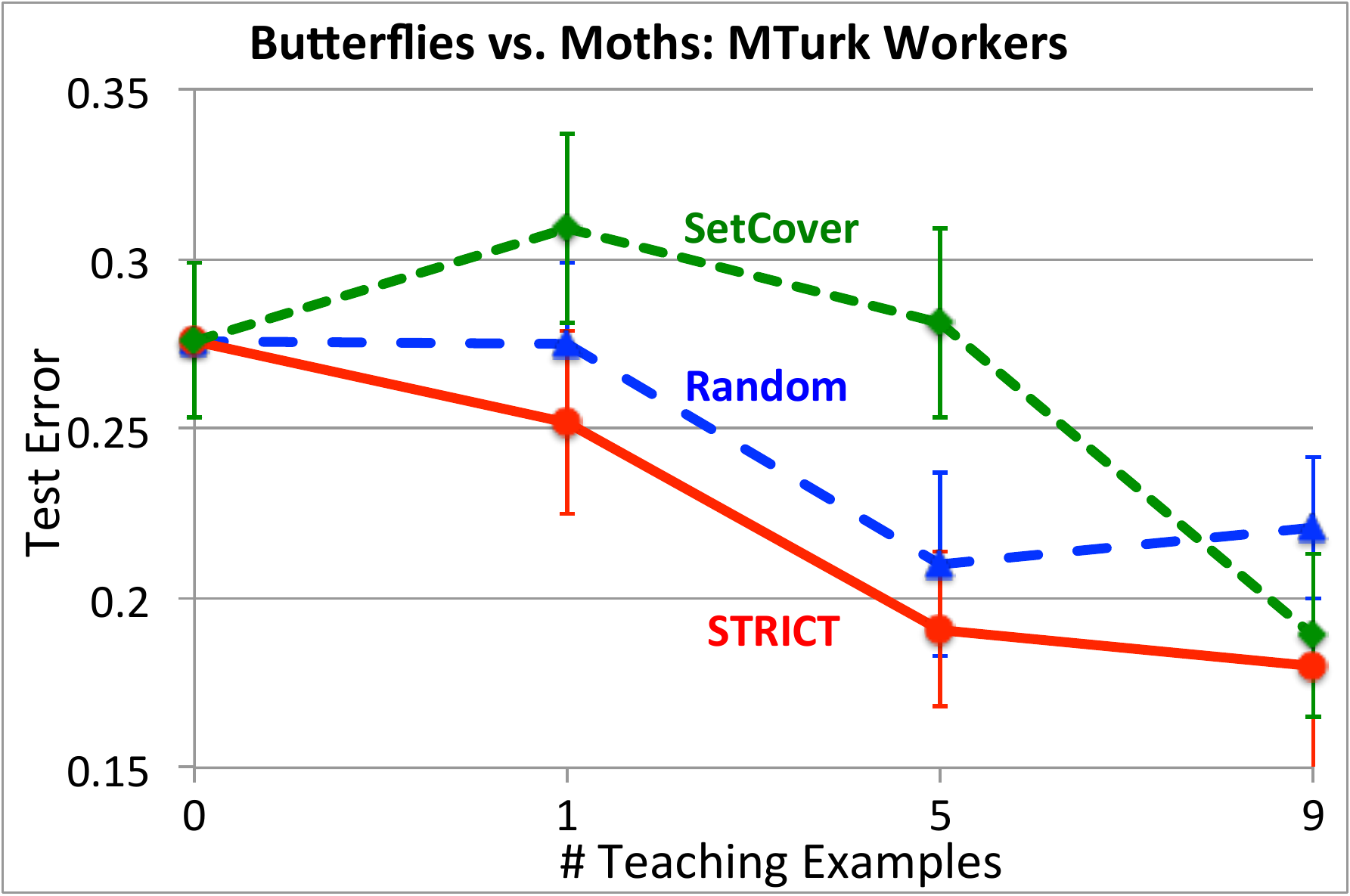}
     \label{fig:experiment-set-BM-MTurk}
   }
   \subfigure[Test error on WP dataset]{
     \includegraphics[width=0.31\textwidth]{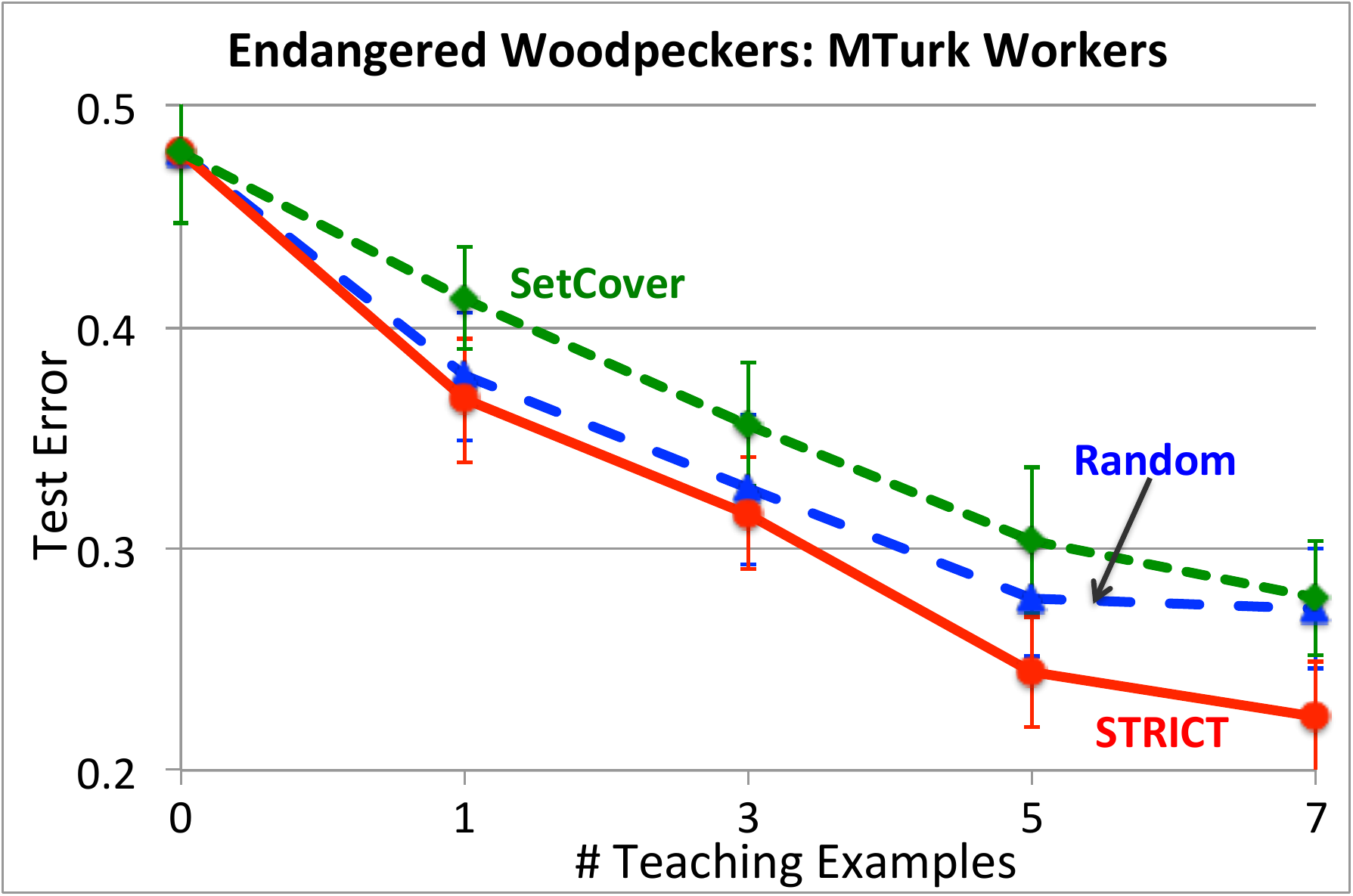}
     \label{fig:experiment-set-WP-MTurk}
   }
\vspace{-5mm}
\caption{(a) shows the order of examples picked by the teaching algorithms.
% for showing to simulated as well as MTurk workers. 
For the VW and BM tasks, we embed the examples in the 2D-feature space in Figs.~\ref{fig:vw_picked_HX} and~\ref{fig:bm_picked_HX}. 
(b-d) %Figs.~\ref{fig:experiment-set-VW-MTurk}, \ref{fig:experiment-set-BM-MTurk}, and \ref{fig:experiment-set-WP-MTurk} 
show the teaching performance of our algorithm measured in terms of test error of humans learners (MTurk workers) on hold out data. \algo is compared against {\em SetCover} and {\em Random} teaching, as we vary the length of teaching.
}
\vspace{-5mm}
\label{fig:results-mturk}
\end{figure*}
%%%%%%%%%%%%%%%%%%%%%%%%%%%%%%%%%%%%%%%%%%%%%%%%%%
%%%%%%%%%%%%%%%%%%%%%%%%%%%%%%%%%%%%%%%%%%%%%%%%%%%%%%%%%%%%%%%%%%%%%%%%%%
\vspace{-7mm}
\subsection{Results on Simulated Learners}
\vspace{-1mm}
We start with simulated learners and report results only on the VW dataset here for brevity. The simulations allow us to control the problem (parameters of learner, size of hypothesis space, \emph{etc.}), and hence gain more insight into the teaching process. Additionally, we can observe how robust our teaching algorithm is against misspecified parameters. 
%, as the actual values of $\alpha$ of the learner in general will not be known by the teacher.

\vspace{-1mm}
{\bf Test error.}  We simulated 100 learners with varying $\alpha$ parameters chosen randomly from the set $\{2,3,4\}$ and different initial hypotheses of the learners, sampled from $\Hyp$. We varied the experimental setting by changing the size of the hypothesis space and the $\alpha$ value used by \algo. Fig.~\ref{fig:experiment-set-VW-Syn-Err} reports results with $\alpha = 2$ for \algo and size of hypothesis class $96$ (2 hypotheses per each of the eight clusters, described in Section~\ref{sec:experiments_setup} for the VW dataset.

\vspace{-1mm}
{\bf How robust is \algo for a mismatched $\alpha$?} In real-world annotation tasks, the learner's $\alpha$ parameter is not known. In this experiment, we vary the $\alpha$ values used by the teaching algorithm \algo against three learners with values of $\alpha=1, 2 \text{ and } 3$.  Fig.~\ref{fig:experiment-set-VW-Syn-VaryingAlpha} shows that a conservative teacher using $\alpha$ bounded in the range $1$ to $5$ performs as good as the one knowing the true $\alpha$ value.

\vspace{-1mm}
{\bf On the difficulty level of teaching.}  Fig.~\ref{fig:experiment-set-VW-Syn-Difficulty} shows the difficulty of examples picked by different algorithms during the process of teaching, where difficulty is measured in terms of expected uncertainty (entropy) that a learner would face for the shown example, assuming that the expectation is taken w.r.t. the learners current posterior distribution over the hypotheses. {\em SetCover}  starts with  difficult examples assuming that the learner is perfect.  \algo starts with easy examples, followed by more difficult ones, as also illustrated in the experiments in Fig.~\ref{fig:all_picked}. Recent results of \citet{2013-aaai_teaching} show that such curriculum-based learning (where the difficulty level of teaching increases with time) indeed is a useful teaching mechanism. Note that our  teaching process inherently incorporates this behavior, without requiring explicit heuristic choices. Also, the transition of {\em SetCover}  to easier examples is just an artifact as {\em SetCover} randomly starts selecting examples once it (incorrectly) infers that the learner has adopted the target hypothesis.  The difficulty can be easily seen when comparing the examples picked by {\em SetCover} and \algo in Fig.~\ref{fig:all_picked}.
%%%%%%%%%%%%%%%%%%%%%%%%%%%%%%%%%%%%%%%%%%%%%%%%%%%%%%%%%%%%%%%%%%%%%%%%%%
%%%%%%%%%%%%%%%%%%%%%%%%%%%%%%%%%%%%%%%%%%%%%%%%%%%%%%%%%%%%%%%%%%%%%%%%%%
\vspace{-2mm}
\subsection{Results on MTurk Workers}
%\subsection{Results on Real Image Annotation Tasks on Amazon Mechanical Turk}
\vspace{-1mm}
Next, we measure the performance of our algorithms when deployed on the actual MTurk platform.

\vspace{-1mm}
{\bf Generating the teaching sequence.} 
%In order to teach the workers on MTurk, we need to generate the teaching sequences by running our algorithms on simulated learners. 
We generate sequences of teaching examples for \algo, as well as {\em Random} and {\em SetCover}.
%We generated the teaching sequence, as shown in 
%Figure~\ref{fig:all_picked} shows the sequences \algo picks on all three data sets.
%, on the simulated learners on all the three datasets and all algorithms. 
We used the feature spaces $\X$ and hypothesis spaces $\Hyp$ as explained in Section~\ref{sec:experiments_setup}. We chose $\alpha=2$ for our algorithm \algo.
% and for the simulated learners.  
To better understand the execution of the algorithms, we illustrate the examples picked by our algorithm as part of teaching, shown in Fig.~\ref{fig:all_picked}. %\ref{fig:experiment-set-VW-Syn-Difficulty}. 
We further show these examples in the 2-D embedding for the VW and BM datasets 
%for  VW and BM datasets 
in Figs.~\ref{fig:vw_picked_HX} and~\ref{fig:bm_picked_HX}. %, respectively.

%These teaching sequences were then used to teach the workers on MTurk using \algo.

\vspace{-1mm}
{\bf Workers on MTurk and the teaching task.}  We recruited workers from the MTurk platform by posting the tasks on the MTurk platform.
% and carried the task of teaching these workers using a webservice hosted on our servers. 
Workers were split into different control groups, depending on the algorithm and the length of teaching used (each control group corresponds to a point in the plots of Fig.~\ref{fig:results-mturk}). Fig.~\ref{fig:hit} provides a high level overview of how the teaching algorithm interacted with the worker. Teaching is followed by a phase of testing examples without providing feedback, for which we report the classification error. For the VW dataset, a total of 780 workers participated (60 workers per control group). For BM, a total of 300 workers participated, and 520 participated in the WP task. The length of the teaching phase was varied as shown in Fig.~\ref{fig:results-mturk}. The test phase was set to 10 examples for the VW and BM tasks, and 16 examples for the WP task. The workers were given a fixed payment for participation and completion, additionally a bonus payment was reserved for the top 10\% performing workers within each control group.

%A total of 210 workers participated in the Vespula identification task with testing phase of 15 examples, and 420 workers in the Butterfly identification task with testing phase of 8 examples. This corresponds to running a particular algorithm for a particular number of examples for 30 workers for Vespula and 60 workers for Butterfly identification.

%We seek to understand the following two questions:
\vspace{-2mm}
{\bf Does teaching help?} \looseness -1 Considering the worker's test set classification performance
% on test classification error on all the three datasets 
in Fig.~\ref{fig:results-mturk}, we can consistently see 
%qualitatively the same improvement in 
an accuracy improvement as workers classify unseen images. This aligns with the results from simulated learners and shows that teaching is indeed helpful in practice. Furthermore, the improvement is monotonic w.r.t. the length of teaching phase used by \algo. In order to understand the significance of these results, we carried out Welch's t-test
% (t-test of unpaired samples, of possibly unequal variance) %\footnote{\url{http://en.wikipedia.org/wiki/Welch's_t_test}}
comparing the workers who received teaching by \algo to the control group of workers without any teaching. The hypothesis that \algo significantly improves the classification accuracy has two-tailed p-values of $p<0.001$ for VW and WP tasks, and $p=0.01$ for the BM task.

{\bf Does our teaching algorithm outperform baselines?} Fig.~\ref{fig:results-mturk} demonstrates that our algorithm \algo outperforms both {\em Random} and {\em SetCover} teaching qualitatively in all studies. We check the significance by performing a paired-t test, %\footnote{\url{http://en.wikipedia.org/wiki/Student's_t-test}}
by computing the average performance of the workers in a given control group and pairing the control groups with same length of teaching for a given task. For the VW task, \algo is significantly better than {\em SetCover} and {\em Random} (at %with corresponding two-tailed p-values at 
$p=0.05$ and $p=0.05$). 
For WP, \algo is significantly better than SetCover (%two-tailed 
%p-value 
$p= 0.002$) whereas comparing with {\em Random}, the p-value is $p=0.07$.
%%%%%%%%%%%%%%%%%%%%%%%%%%%%%%%%%%%%%%%%%%%%%%%%%%%%%%%%%%%%%%%%%%%%%%%%%%
%%%%%%%%%%%%%%%%%%%%%%%%%%%%%%%%%%%%%%%%%%%%%%%%%%%%%%%%%%%%%%%%%%%%%%%%%%

%%%%%%%%%%%%%%%%%%%%%%%%%%%%%%%%%%%%%%%%%%%%%%%%%%%%%%%%%%%%%%%%%%%%%%%%%%
%%%%%%%%%%%%%%%%%%%%%%%%%%%%%%%%%%%%%%%%%%%%%%%%%%%%%%%%%%%%%%%%%%%%%%%%%%
% !TEX root =  teaching.tex
%%%%%%%%%%%%%%%%%%%%%%%%%%%%%%%%%%%%%%%%%%%%%%%%%%%%%%%%%%%%%%%%%%%%%%%%%%
%%%%%%%%%%%%%%%%%%%%%%%%%%%%%%%%%%%%%%%%%%%%%%%%%%%%%%%%%%%%%%%%%%%%%%%%%%
\vspace{-2mm}
\section{Conclusions}
\vspace{-1mm}
%\gabor{New conclusions needed}
We proposed a noise-tolerant stochastic model of the workers' learning process in crowdsourcing classification tasks. 
%Our model of the learner captures some natural aspects such as idiosyncratic prior knowledge, % , like initial skills (such as pre-knowledge, cultural or geographical biases),
%reliability (to capture competence) and eagerness (to capture the ability to learn quickly). 
We then developed a novel teaching algorithm \algo that exploits this model to teach the workers efficiently. Our model generalizes existing models of teaching in order to increase robustness. We proved strong theoretical approximation guarantees on the convergence to a desired error rate. Our extensive experiments on simulated workers as well as on three real annotation tasks on the Mechanical Turk platform demonstrate the effectiveness of our teaching approach.

More generally, our approach goes beyond solving the problem of teaching workers in crowdsourcing services. With the recent growth of online education and tutoring systems  \footnote{c.f., \url{https://www.coursera.org/}}, algorithms such as \algo can be envisioned to aid in supporting data-driven online education \cite{2012-hcomp_horvitz_personalized-online-education,2013-chi_crowd-based-classroom}.
%More generally, our approach goes much beyond just solving the problem of teaching workers in crowdsourcing services. Adaptive teaching systems can be used in providing personalized and adaptive online education on the internet. We believe our work is a step forward in envisioning such intelligent teaching systems.
%%%%%%%%%%%%%%%%%%%%%%%%%%%%%%%%%%%%%%%%%%%%%%%%%%%%%%%%%%%%%%%%%%%%%%%%%%
%%%%%%%%%%%%%%%%%%%%%%%%%%%%%%%%%%%%%%%%%%%%%%%%%%%%%%%%%%%%%%%%%%%%%%%%%%

%%%%%%%%%%%%%%%%%%%%%%%%%%%%%%%%%%%%%%%%%%%%%%%%%%%%%%%%%%%%%%%%%%%%%%%%%%
%%%%%%%%%%%%%%%%%%%%%%%%%%%%%%%%%%%%%%%%%%%%%%%%%%%%%%%%%%%%%%%%%%%%%%%%%%
%\clearpage
%\addtocounter{section}{4}
\bibliographystyle{icml2014}
\bibliography{teaching}  % sigproc.bib is the name of the Bibliography in this case
% You must have a proper ".bib" file
%  and remember to run:
% latex bibtex latex latex
% to resolve all references
%
% ACM needs 'a single self-contained file'!
%%%%%%%%%%%%%%%%%%%%%%%%%%%%%%%%%%%%%%%%%%%%%%%%%%%%%%%%%%%%%%%%%%%%%%%%%%
%%%%%%%%%%%%%%%%%%%%%%%%%%%%%%%%%%%%%%%%%%%%%%%%%%%%%%%%%%%%%%%%%%%%%%%%%%
%ACKNOWLEDGMENTS are optional
%\section{Acknowledgments}
%Add Acknowledgments.
%%%%%%%%%%%%%%%%%%%%%%%%%%%%%%%%%%%%%%%%%%%%%%%%%%%%%%%%%%%%%%%%%%%%%%%%%%
%%%%%%%%%%%%%%%%%%%%%%%%%%%%%%%%%%%%%%%%%%%%%%%%%%%%%%%%%%%%%%%%%%%%%%%%%%
%\balancecolumns
% !TEX root =  teaching.tex

\appendix
\onecolumn
\section{Supplementary Material}
\subsection{Proofs}
\begin{proof}[Proof of Proposition~\ref{prop:hardness}]
We reduce from set cover. Suppose we are given a collection of finite sets $S_1,\dots,S_n$ jointly covering a set $W$. We reduce the problem of finding a smallest subcollection covering $W$ to the teaching problem with the special case $\alpha=\infty$.

Let $\Hyp=W\cup\{h^*\}$, that is, each element in $W$ is a hypothesis that misclassifies at least one data point. We use a uniform prior $p(h)=\frac{1}{|W|+1}$. For each set $S_j$, we create a teaching example $x_j$. The label output by hypothesis $h(x_j)=1$ iff $h\in S_j$, otherwise $h(x_j)=-1$. We set $h^*(x)=-1$ for all examples. Thus, selecting $S_i$ in the set cover problem is equivalent to selecting example $x_i$. It is easy to see that constructing the examples can be done in polynomial (in fact, linear) time.

The expected error after showing a set of examples is less than $\frac{1}{(|W|+1)n}$ if and only if sets indexed by $A$ cover $W$. Thus, if we could efficiently find the smallest set $A$ achieving error less than $\frac{1}{(|W|+1)n}$, we could efficiently solve set cover.
\end{proof}

Before proving the main theorems, we state an important lemma that will be needed throughout the analysis.
\begin{lemma}\label{lem:stoc-proc}
  Assume that the learner's current hypothesis $h_t$ is governed by the stochastic process described in Section~\ref{sec:learner}. Then, the marginal distribution of $h_t$ is given by $P_{t-1}(h)$ in every time step $t$.
\end{lemma}
\begin{proof}
  Let the marginal distribution of $h_t$ denoted by $P'_{t-1}(h)$. We will show by induction that for every $t$, $P'_t=P_t$.

  Obviously, $P'_0=P_0$ by definition. Now, as for the induction hypothesis, let us assume that $P'_{t-1}=P_{t-1}$. By the definition of the stochastic process we have
  \begin{align*}
    P'_t(h)&=\frac{1}{Z'_t}\left(P'_{t-1}(h)\mathbb{I}\{y_t=h(x_t)|h,x_t\} + P_{t}(h)\mathbb{I}\{y_t\neq h(x_t)|h,x_t\}\right)\\
    &=\frac{1}{Z'_t}\left(P_{t-1}(h)\mathbb{I}\{y_t=h(x_t)|h,x_t\} + P_{t-1}(h)P(y_t|h,x_t)\mathbb{I}\{y_t\neq h(x_t)|h,x_t\}\right)\\
    &=\frac{1}{Z'_t}P_{t-1}(h)\left(\mathbb{I}\{y_t=h(x_t)|h,x_t\}+P(y_t|h,x_t)\mathbb{I}\{y_t\neq h(x_t)|h,x_t\}\right)\\
    &=\frac{1}{Z'_t}P_{t-1}(h)P(y_t|h,x_t)^{\mathbb{I}\{y_t\neq h(x_t)|h,x_t\}}=P_t(h)\,,
  \end{align*}
  as stated.
\end{proof}
\begin{proof}[Proof of Theorem~\ref{thm:UGTP}]
  Clearly, $F(A)$ can be written as
\begin{align*}
  F(A) &= \sum_{h\in \Hyp}P_0(h)G_h(A)\err(h,h^*)\,,\\
  \intertext{where}
  G_h(A) &= 1-\prod_{\substack{x\in A\\y(x)\neq\sgn(h(x))}}P(y(x)|h,x)\,.
\end{align*}
It is easy to see that $G_h(A)$ is submodular for every $h\in \Hyp$. Thus, $F(A)$ is also submodular.

Let us start to upper bound the expected error of the learner. For that, we need the following simple observation:
\begin{align*}
\frac{P(h|A)}{P(h^*|A)} &=\frac{Q(h|A)}{Q(h^*|A)}=\frac{Q(h|A)}{P_0(h^*)}\,.
\end{align*}
Now for the upper bounding:
\begin{align*}
  \sum_{h\in \Hyp}P(h|A)\err(h,h^*)&\leq \sum_{h\in \Hyp}\frac{P(h|A)}{P(h^*|A)}\err(h,h^*)\\
  &= \frac{1}{P_0(h^*)}\sum_{h\in \Hyp}Q(h|A)\err(h,h^*)\\
  &=\frac{1}{P_0(h^*)}(E-F(A))\,,
\end{align*}
where $E=\sum_{h\in \Hyp}P_0(h)\err(h,h^*)$ is an upper bound on the maximum of $F(A)$. This means that if we choose a subset $A$ such that $F(A)\geq E-P_0(h^*)\epsilon$, it guarantees an expected error less than $\epsilon$. In the following, we assume that $F(\X)\geq E-P_0(h^*)\epsilon/2$. If this assumption is violated, the Theorem still holds, but the bound is meaningless, since $\OPT(P_0(h^*)\epsilon/2)=\infty$ in this case.

Since $F(A)$ is submodular (and monotonic), we can achieve $E-P_0(h^*)\epsilon$ ``level'' with the greedy algorithm, as described below. We use the following result of the greedy algorithm for maximizing submodular functions:
\begin{theorem*}[\citet{krause12survey}, based on \citet{1978-_nemhauser_submodular-max}]
  Let $f$ be a nonnegative monotone submodular function and let $S_t$ denote the set chosen by the greedy maximization algorithm after $t$ steps. Then we have
  \begin{align*}
    f(S_\ell) &\geq \left(1-e^{-\ell/k}\right)\max_{S:|S|= k}f(S)
  \end{align*}
  for all integers $k$ and $l$.
\end{theorem*}
Let $k^*$ be the cardinality of the smallest set $A^*$ such that $F(A^*)\geq E-P_0(h^*)\epsilon/2$. Thus we know that
\begin{align*}
  \max_{A:|A|= k^*}F(A) &\geq E-P_0(h^*)\epsilon/2\,.
\end{align*}
Now we set $\ell=k^*\log\frac{2E}{P_0(h^*)\epsilon}$ and we denote $A_\ell$ the result of the greedy algorithm after $\ell$ steps, and we get
\begin{align*}
  F(A_\ell) &\geq \left(1-e^{-l/k^*}\right)\left(E-\frac{P_0(h^*)\epsilon}{2}\right)\\
  &=\left(1-\frac{P_0(h^*)\epsilon}{2E}\right)\left(E-\frac{P_0(h^*)\epsilon}{2}\right)\\
  &\geq E-p(h^*)\epsilon\,,
\end{align*}
proving that running the greedy algorithm for $\ell$ steps achieves the desired result.

\end{proof}

\begin{proof}[Proof of Theorem~\ref{thm:stronger}]
 We introduce a randomized teaching policy called Relaxed-Greedy Teaching Policy (sketched in Policy~\ref{alg:RGTP}) and prove that with positive probability, the policy reduces the learner error exponentially. Then, we use the standard probabilistic argument: positive probability of the above event implies that there must exist a sequence of examples that reduce the learner error exponentially. We finish the proof of the theorem by using the result of Theorem~\ref{thm:UGTP}.

 Based on our model, the way the learner updates his/her belief after showing example $x_t\in \X$ and receiving answer $y_t =\sgn( h^*(x_t))$ is as follows:
$$\pl{h}{t+1} = \frac{1}{\NF{t}}\pl{h}{t} \wl^{(1-\xi_t(h))/2},$$
where  $\xi_t(h) =\sgn(h(x_t))\cdot y_t$, the term $\NF{t}$ is the normalization factor, and $0<\wl<1$ is a parameter by which the learner decreases the weight of inconsistent hypotheses. Note that $\wl$ may very well depend on the examples shown, i.e., for hard examples $\wl$ is typically larger than those of the easy ones as the learner is more certain about his/her answers. However, here we assume that $\wl\leq \wt<1$ and that $\wt$ is known to the teacher. In other words, the teacher knows the minimum weight updates imposed by the learner on inconsistent hypotheses.  As a result, the teacher can track $\pl{h}{t+1}$ conservatively as follows:
\begin{equation}\label{teacher_update}
\pt{h}{t+1} = \frac{1}{\NF{t}}\pt{h}{t} \wt^{(1-\xi_t(h))/2}.
\end{equation}
\begin{theorem}\label{thm:RGTP}
Let $\Hyp$ be a collection of $n$ linear separators and choose an $0< \epsilon< 1$. Then, under the condition that $\X$ is $m$-rich, \Alggbs guarantees to achieve
$$\Pr(1-\pl{h^*}{m}>\epsilon)<\frac{(1-\epsilon)(1-p_0(h^*))}{\epsilon \cdot p_0(h^*)}e^{-m(1-\wt)/4},$$
by showing $m$ examples in total.
In other words, to have $\Pr(1-\pl{h^*}{m}>\epsilon)<\delta$, \Alggbs uses at most the following number of examples:
$$ m = \frac{4}{1-\wt}\log\frac{(1-\epsilon)(1-p_0(h^*))}{\delta\cdot \epsilon\cdot p_0(h^*)}.$$
\end{theorem}
The above theorem requires that $\X$ gets a richer space for obtaining better performance. When we have a uniform prior $\prior=1/n$, the the above bounds simplify to
$$ m = \frac{4}{1-\wt}\log\frac{(1-\epsilon)n}{\delta\cdot \epsilon}.$$
As at least $\log n$ queries is required to identify the correct hypothesis with probability one, the above bound is within a constant factor from $\log n$ for fixed $\epsilon$ and $\delta$.

\begin{algorithm}[tb]
   \caption{Relaxed-Greedy Teaching Policy (\Alggbs)}\label{alg:RGTP}
\begin{algorithmic}[1]
   \STATE{{\bfseries Input:} examples $\X$, hypothesis $\Hyp$, prior $\prior$, error $\epsilon$.}
   \STATE{$t=0, \pt{h}{0}=\prior$}
   \WHILE{$1-\pt{h^*}{t}>\epsilon$}
   \IF {there exists two neighboring polytopes $\poly$ and $\poly'$ s.t. $\sum_h \pt{h}{t}h(\poly)>0$ and  $\sum_h \pt{t}{h} h(\poly')<0$}
   \STATE select $x_t$ uniformly at random from $\poly$ or $\poly'$.
   \ELSE
   \STATE select $x_t$ from polytop $\poly=\argmin_{\poly\in \polytop} |\sum_h\pt{h}{t} h(\poly)|$
   \ENDIF
   \STATE $\forall h\in\Hyp$ update $\pt{h}{t+1}$ according to \eqref{teacher_update} and $t\rightarrow t+1$.
   \ENDWHILE
\end{algorithmic}
\end{algorithm}

The proof technique is inspired by \cite{Burnashev74}, \cite{Karp07}, and in particular  beautiful insights in \cite{Nowak11}.
To analyze \Alggbs let us define the random variable
 $$\BZ^{(l)}_t =\frac{1-\pl{h^*}{t}}{\pl{h^*}{t}}.$$
This random variable $\log(\BZ_t)$ was first introduced by \cite{Burnashev74} in order to analyze the classic binary search under noisy observations (for the ease of exposure we use $\BZ_t$ instead of $\log(\BZ_t)$). It basically captures the probability mass put on the incorrect hypothesis after $t$ examples. Similarly, we can define
 $$\BZ^{(t)}_t =\frac{1-P^{(t)}_t(h^*)}{p^{(t)}_t(h^*)}.$$
A simple fact to observe is the following lemma.
\begin{lemma}
For any sequence of examples/labels $\{(x_t,y_t)\}_{t\geq0}$, and as long as $0\leq \wl\leq \wt\leq1$ we have $\BZ^{(l)}_t\leq  \BZ^{(t)}_t$.
\end{lemma}
 Note that \Alggbs is a randomized algorithm. Using Markov's inequality we obtain
 \begin{align*}
 \Pr(1-\pl{h^*}{m})>\epsilon) &\leq  \Pr(1-\pt{h}{h^*})>\epsilon)\\
  &=\Pr\left(\BZ_t>\frac{\epsilon}{1-\epsilon}\right)\\
 &\leq\frac{(1-\epsilon)\E(\BZ^{(t)}_t)}{\epsilon}.
 \end{align*}
%$$\Pr(1-p_i(h^*)>\epsilon) = \Pr\left(\BZ_i>\frac{\epsilon}{1-\epsilon}\right)\leq \frac{(1-\epsilon)\E(\BZ_i)}{\epsilon}.$$
The above inequalities simply relates the probability we are looking for in Theorem~\ref{alg:RGTP} to the expected value of $\BZt_t$. Hence, if we can show that the expected value decreases exponentially fast, we are done. To this end, let us first state the following observation.
\begin{lemma}
For any sequence of examples/labels $\{(x_t,y_t)\}_{t\geq0}$, and for $0\leq\wt < 1$ the corresponding random variable $\{\BZt_t\}_{t\geq0}$ are all non-negative and decreasing, i.e.,
$$0 \leq\BZt_s\leq \BZt_t\leq\frac{1-P_0(h^*)}{P_0(h^*)}, \quad s\geq t.$$
\end{lemma}
The above lemma simply implies that the sequence $\{\BZ_t\}_{t\geq0}$ converges. However, it does not indicate the rate of convergence. Let us define  $\filter_t = \sigma(P_0^{(t)}, p_1^{(t)}, \dots,p_t^{(t)})$ the sigma-field generated by random variables $P_0^{(t)}, p_1^{(t)}, \dots,p_t^{(t)}$. Note that $\BZt_t$ is a function of $p_t^{(t)}$ thus $\filter_t$-measurable. Now, by using the towering property of the expectation we obtain
 \begin{displaymath}
 \E(\BZt_t) = \E((\BZt_t/\BZt_{t-1})\BZt_{t-1} ) = \E(\E((\BZt_t/\BZt_{t-1})\BZt_{t-1}|\filter_{t-1}))
 \end{displaymath}
 Since $\BZt_{t-1}$ is $\filter_{t-1}$-measurable we get
 \begin{align*}
 \E(\BZt_t) &=  \E(\BZt_{t-1}\E((\BZt_t/\BZt_{t-1})|\filter_{t-1}))\\
  &\leq \E(\BZt_{t-1}) \max_{\filter_{t-1}} \E((\BZt_t/\BZt_{t-1})|\filter_{t-1}).
 \end{align*}
 The above inequality simply implies that
 \begin{equation}
  \E(\BZt_t) =  \frac{1-P_0(h^*)}{P_0(h^*)} \left(\max_{0\leq s\leq t-1}\max_{\filter_{s}} \E((\BZt_{s+1}/\BZt_{s})|\filter_{s})\right)^t
 \end{equation}
 In the remaining of the proof we derive a uniform upper bound (away from 1) on $\E((\BZt_t/\BZt_{t-1})|\filter_{t-1})$, which readily implies exponential decay on $\Pr(1-\pl{h^*}{m}>\epsilon)$ as the number of samples $m$ grows.  For the ease of presentation, let us define the (weighted) proportion of hypothesis that agree with $y_t$ as follows:
 $$\wph_t = \frac{1}{2}\left(1+\sum_{h}\pt{h}{t}\xii(h)\right).$$
 Along the same line, we define the proportion of hypothesis that predict $+$ on polytope $\poly$ as follows
  $$\ppp{\poly} = \frac{1}{2}\left(1+\sum_{h}\pt{h}{t}h(\poly)\right).$$
 Now, we can easily relate $\wph_t$ to the normalization factor $\NF{t}$:
 $$\NF{t}=\sum_h\pt{h}{t} \wt^{(1-\xii(h))/2} = (1-\wph_t)\wt+\wph_t.$$
 As a result
 $$\pt{h}{t+1} = \frac{\pl{h}{t} \wt^{(1-\xii(h))/2}}{(1-\wph_t)\wt+\wph_t}.$$
 In particular for $\pt{h^*}{t+1}$ we have
 $$\pt{h^*}{t+1} = \frac{\pl{h}{t}}{(1-\wph_t)\wt+\wph_t}.$$
 To simplify the notation, we define
 $$\gamma_t =(1-\wph_t)\wt+\wph_t.$$
 Hence,
 $$\frac{\BZt_{t+1}}{\BZt_{s}} = \frac{\gamma_t-\pt{h^*}{t}}{1-\pt{h^*}{t}}.$$
 Note that since $\pt{h^*}{t}$ is $\filter_{t}$-measurable the above equality entails that
 $$\E\left(\frac{\BZt_{t+1}}{\BZt_{t}}|\filter_t \right)= \frac{\E(\gamma_t|\filter_t)-\pt{h^*}{t}}{1-\pt{h^*}{t}} .$$
 Thus we need to show that $\E(\gamma_t|\filter_t)$ is bounded away from 1. To this end, we borrow the following geometric lemma from \cite{Nowak11}.
 % In the following (and in order to be consistent with \cite{Nowak11}) we  define the weighted prediction
 \begin{lemma}\label{lem:bipolar}
Let $\Hyp$ consists of a set of linear separators where each induced polytope $\poly\in\polytop$  contains at least one example $x\in\X$. Then for any probability distribution $p$ on $\Hyp$ one of the following situations happens
\begin{enumerate}
\item  either there exists a polytope $\poly$ such that $\sum_{h} p(h)h(\poly)=0$, or
\item there exists a pair of neighboring polytopes $\poly$ and $\poly'$ such that $\sum_{h} p(h)h(\poly)>0$ and $\sum_{h} p(h)h(\poly')<0$.
\end{enumerate}
\end{lemma}
The above lemma essentially characterizes Ham Sandwich Theorem \cite{ham94} in discrete domain $\X$ that is $1$-rich. In words, Lemma~\ref{lem:bipolar} guarantees that either there exists a polytope where (weighted) hypothesis greatly disagree, or there are two neighboring polytopes that are bipolar. In either case, if an example is shown randomly from these polytopes, it will be very informative. This is essentially the reason why \Alggbs performs well.

Now, let $\poly_t$ be the polytope from which the example $x_t$ is shown. Then, based on $y_t$ we have two cases:
 \begin{itemize}
 \item if $y_t=+$ then $\gamma_t^+\doteq\gamma_t = (1-\ppp{\poly_t})\wt+\ppp{\poly_t}$,
 \item if $y_t=-$ then $\gamma_t^- \doteq\gamma_t = \ppp{\poly_t}\wt+1-\ppp{\poly_t}$.
 \end{itemize}
 Note that for any $x_t$ picked by \Alggbs we have $0<\ppp{\poly_t}<1$, since it never shows an example that all hypothesis agree on. As a result, both $\gamma_t^+$ and $\gamma_t^-$ are between $0$ and $1$.

Based on Lemma~\ref{lem:bipolar} there are only two cases. Let us define the auxiliary random variable $s_t$ that simply indicate in which case we are. More precisely, $s_t=1$ indicates that we are in case 1 and $s_t=2$ indicates that we are in case 2.
To be formal we define $\filterg_t =\sigma(P_0^{(t)}, p_1^{(t)}, \dots,p_t^{(t)}, s_t)$. Note that $\filter_t\subset\filterg_t$ and thus $\E(\gamma_t|\filter_t)=\E(\E(\gamma_t|\filterg_t)|\filter_t)$. We need to prove the following technical lemma.
\begin{lemma}\label{lem:g}
\begin{align*}
&\E(\gamma_t|\filterg_t) \\
&\leq \max\left\{\frac{3+\wt}{4}, \frac{1+\wt}{2}, 1-\frac{(1-\wt)(1-\pt{h^*}{t})}{2}  \right\}.
\end{align*}
\end{lemma}
\begin{proof}
 Let us first condition on $s_t=1$. Then, \Alggbs chooses an $x_t\in\poly_t$ in which case $\ppp{\poly_t}=1/2$ and results in $\gamma_t^+ = \gamma_t^- = (\wt+1)/2$. Hence, given $s_t=1$, we have
\begin{equation}
\E(\gamma_t|\filterg_t) = (\wt+1)/2.
\end{equation}

The conditioning on $s_t=2$ is a little bit more elaborate. Recall that in this case \Alggbs randomly chooses one of $\poly$ and $\poly'$. Note that $\ppp{\poly}>1/2$ and $\ppp{\poly'}<1/2$.  Now we encounter 4 possibilities:
\begin{enumerate}
\item $h^*(\poly)=h^*(\poly')=+$: condition on $s_t=2$ we have
\begin{align}
\E(\gamma_t|\filterg_t) &= \frac{\gamma_t^+ + \gamma_t^-}{2}\nonumber\\
&\leq \frac{1 +  (1-\ppp{\poly'})\wt+\ppp{\poly'}}{2}\nonumber\\
&\leq \frac{3+\wt}{4}\label{firs}
\end{align}
where we used the fact that $\gamma_t^+\leq1$ and $(1-\ppp{\poly'})\wt+\ppp{\poly'}$ is an increasing function of
$\ppp{\poly'}$ and that $\ppp{\poly'}<1/2$.

\item $h^*(\poly)=h^*(\poly')=-$: similar argument as above shows that
\begin{equation*}
\E(\gamma_t|\filterg_t)\leq \frac{3+\wt}{4}.
\end{equation*}
\item $h^*(\poly)=-, h^*(\poly')=+$: In this case we have
\begin{align}
\E(\gamma_t|\filterg_t) &= \frac{\gamma_t^+ + \gamma_t^-}{2}\nonumber\\
&=\frac{\ppp{\poly}\wt+1-\ppp{\poly}+  (1-\ppp{\poly'})\wt+\ppp{\poly'}}{2}\nonumber\\
&= 1-\frac{1-\wt}{2}(1+\ppp{\poly}-\ppp{\poly'})\nonumber\\
&\leq \frac{1+\wt}{2}\label{second}
\end{align}
where we used the fact $0\leq \ppp{\poly}-\ppp{\poly'}\leq 1$.
\item $h^*(\poly)=+, h^*(\poly')=-$: since $\poly$ and $\poly'$ are neighboring polytopes, $h^*$ should be the common face. Hence, we have $\ppp{\poly}-\ppp{\poly'} = \pt{h^*}{t}$. As a result
\begin{align}
\E(\gamma_t|\filterg_t) &= \frac{\gamma_t^+ + \gamma_t^-}{2}\nonumber\\
&=\frac{(1-\ppp{\poly})\wt+\ppp{\poly}+  \ppp{\poly'}\wt+1-\ppp{\poly'}}{2}\nonumber\\
&= \frac{1+\ppp{\poly}-\ppp{\poly'}+\wt(1- \ppp{\poly}+\ppp{\poly'})}{2}\nonumber\\
&\leq 1-\frac{(1-\wt)(1-\pt{h^*}{t})}{2}.\label{third}
\end{align}
\end{enumerate}
By combining \eqref{firs}, \eqref{second} and \eqref{third} we prove the lemma.
\end{proof}

Lemma~\ref{lem:g} readily implies that
\begin{align*}
\E\left(\frac{\BZt_{t+1}}{\BZt_{t}}|\filter_t \right)&= \frac{\E(\gamma_t|\filter_t)-\pt{h^*}{t}}{1-\pt{h^*}{t}}\nonumber\\
 &\leq \frac{3+\wt}{4}.
\end{align*}
Hence,
\begin{align*}
\E(\BZt_t) &=  \frac{1-P_0(h^*)}{P_0(h^*)} \left(1-\frac{1-\wt}{4}\right)^t\\
&\leq \frac{1-P_0(h^*)}{P_0(h^*)} \exp(-t\cdot(1-\wt)/4)
\end{align*}
This finishes the proof of Theorem~\ref{thm:RGTP}. Now, to finish the proof of Theorem~\ref{thm:stronger}, we just set $\delta$ to $1/2$ and use the probabilistic argument mentioned in the beginning of the proof, resulting in an upper bound on $\OPT$. Since we use the logistic likelihood function, $w_o$ can be bounded by $\frac{1}{2}$. Theorem~\ref{thm:UGTP} follows.
 \end{proof}

\end{document}